\documentclass[shortAfour,sageh,doublespace,times]{sagej} %

\pdfoutput=1

\newcommand{\UTIAStitle}{Batch Informed Trees (BIT*): Informed Asymptotically Optimal Anytime Search}
\newcommand{\UTIASauthor}{Gammell et al.}
\usepackage{times}
\usepackage{amsmath,amssymb,amsfonts,amscd,amsthm}
\usepackage{mathtools} %
\usepackage{natbib} %
\usepackage{url}
\usepackage{graphicx}
\usepackage{microtype} %
\usepackage{calc} %

\usepackage[printonlyused,nolist]{acronym} %
\usepackage[usenames,dvipsnames]{xcolor} %
\usepackage[ruled,vlined,linesnumbered]{algorithm2e}  %

\usepackage{paralist} %
\usepackage{placeins} %
\usepackage{afterpage} %
\usepackage{enumitem} %
\setlist{nosep} %

\usepackage{soul}   %

\usepackage{tikz}

\usepackage[pdftex,colorlinks]{hyperref}
\usepackage{hypernat} %
\usepackage{bookmark} %
\usepackage[all]{hypcap} %

\SetAlgoSkip{}

\usetikzlibrary{decorations.pathreplacing,calc}

\hypersetup{%
    pdftitle={\UTIASauthor: \UTIAStitle},
    pdfauthor={\UTIASauthor},
    pdfkeywords={},
    pdfsubject={},
    pdfstartview=FitH,%
    breaklinks=true,%
    colorlinks=true,%
    linkcolor=Black,%
    citecolor=Black,%
    urlcolor=Black,%
    anchorcolor=Black,%
    filecolor=Black,%
    menucolor=Black,%
    runcolor=Black %
}%
\theoremstyle{plain}
\newtheorem{lem}{Lemma}
\newtheorem{thm}[lem]{Theorem} %
\newtheorem{defn}[lem]{Definition} %
\newcommand{\squeezeWidowedWords}{\looseness=-1}
\newcommand{\squeezeWidowedLine}{\enlargethispage*{\baselineskip}\pagebreak}
\newcommand{\escapeParagraph}{\vspace{1.5ex}}

\newcommand{\figureCaptionStyle}[0]{}  %

\newcommand{\algorithmStyle}[0]{\algorithmCapStart} %

\newcommand{\algo}[1]{Alg.~\ref{#1}}
\newcommand{\algos}[2]{Algs.~\ref{#1}--\ref{#2}}
\newcommand{\algoAnd}[2]{Algs.~\ref{#1} and \ref{#2}}
\newcommand{\algoline}[2]{\algo{#1}, Line~\ref{#1:#2}}
\newcommand{\algolines}[3]{\algo{#1}, Lines~\ref{#1:#2}--\ref{#1:#3}}

\newcommand{\citeParenthetical}[2]{\citep[#1;][]{#2}}
\newcommand{\citeParentheticalTwice}[4]{(\citealp[#1;][]{#2} and \citealp[#3;][]{#4})}

\newcommand{\citeAcronymDefn}[2]{\acl{#1} \citeParenthetical{\acs{#1}}{#2}\acused{#1}}
\newcommand{\citeAcronymDefnP}[2]{\aclp{#1} \citeParenthetical{\acs{#1}}{#2}\acused{#1}}
\newcommand{\citeExample}[1]{\citep[e.g.,][]{#1}}
\newcommand{\citeCompare}[1]{\citep[cf.][]{#1}}

\newcommand{\newAlgoLine}[1]{ {\color{BrickRed} #1} }

\let\oldnl\nl%
\newcommand{\skipln}{\renewcommand{\nl}{\let\nl\oldnl}}

\newcommand{\algorithmCapStart}[0]%
{%
    \let\oldhypcapspace\hypcapspace%
    \renewcommand{\hypcapspace}{1.5\baselineskip}%
    \capstart%
    \renewcommand{\hypcapspace}{\let\hypcapspace\oldhypcapspace}%
}
\makeatletter
\let\original@algocf@latexcaption\algocf@latexcaption
\long\def\algocf@latexcaption#1[#2]{%
  \@ifundefined{NR@gettitle}{%
    \def\@currentlabelname{#2}%
  }{%
    \NR@gettitle{#2}%
  }%
  \original@algocf@latexcaption{#1}[{#2}]%
}
\makeatother

\newcommand{\tikzmark}[1]{\tikz[overlay,remember picture] \node (#1) {};}

\newcommand*{\AddNote}[4]%
{%
    \begin{tikzpicture}[overlay, remember picture]
        \draw [decoration={brace,amplitude=0.5em},decorate,thick]
            ($([xshift=5pt]#3)!([yshift=5pt]#1.north)!($([xshift=5pt]#3)-(0,1)$)$) --  
            ($([xshift=5pt]#3)!(#2.south)!($([xshift=5pt]#3)-(0,1)$)$)
                node [align=center, text width=1.0cm, pos=0.5, anchor=center, fill=white] {#4};
    \end{tikzpicture}
}%

\DeclareMathOperator*{\argmin}{arg\,min}

\newcommand{\bbm}{\begin{bmatrix}}
\newcommand{\ebm}{\end{bmatrix}}

\newcommand{\norm}[2]{\left\| #1 \right\|_{#2}}
\newcommand{\card}[1]{\left|#1\right|}

\newcommand{\probSymb}[0]{P}
\newcommand{\prob}[1]{\probSymb\left(#1\right)}

\newcommand{\gammaFunc}[1]{\Gamma\left(#1\right)}
\newcommand{\real}[0]{\mathbb{R}}
\newcommand{\positiveReal}[0]{\real_{\geq0}}

\newcommand{\pair}[2]{\left( #1, #2\right)}
\newcommand{\set}[1]{\left\lbrace #1\right\rbrace}
\newcommand{\setst}[2]{\set{#1\;\;\middle|\;\;#2}}

\newcommand{\seqset}[1]{\left(#1\right)}

\newcommand{\closure}[1]{\mathrm{cl}\left(#1\right)}

\newcommand{\cost}[0]{c}

\newcommand{\solutionCost}[0]{f}
\newcommand{\costToCome}[0]{g}
\newcommand{\costToGo}[0]{h}
\newcommand{\iter}[0]{i}
\newcommand{\counterj}[0]{j}
\newcommand{\counterk}[0]{k}

\newcommand{\batches}[0]{\ell}
\newcommand{\samplesPerBatch}[0]{m}
\newcommand{\dimension}[0]{n}

\newcommand{\totalSamples}[0]{q}
\newcommand{\radius}[0]{r}
\newcommand{\pathCostSymb}[0]{\cost}

\newcommand{\stateu}[0]{\mathbf{u}}
\newcommand{\statev}[0]{\mathbf{v}}

\newcommand{\statew}[0]{\mathbf{w}}
\newcommand{\statex}[0]{\mathbf{x}}
\newcommand{\statey}[0]{\mathbf{y}}

\newcommand{\rrewire}[0]{\radius_{\mathrm{rewire}}}
\newcommand{\rbitstar}[0]{\radius_{\mathrm{BIT}^*}}
\newcommand{\rbitstarmax}[0]{\radius_{\mathrm{max}}}
\newcommand{\rbitstarmin}[0]{\radius_{\mathrm{BIT}^*}^*}
\newcommand{\rrrtstar}[0]{\radius_{\mathrm{RRT}^*}}
\newcommand{\rrrtstarmin}[0]{\radius_{\mathrm{RRT}^*}^*}
\newcommand{\kbitstar}[0]{k_{\mathrm{BIT}^*}}
\newcommand{\kbitstarmin}[0]{k_{\mathrm{BIT}^*}^*}

\newcommand{\maxEdge}[0]{\eta}
\newcommand{\lebesgueSymb}[0]{\lambda}

\newcommand{\sampleDensity}[0]{\rho}
\newcommand{\pathSeq}[0]{\sigma}

\newcommand{\unitBall}[1]{\zeta_{#1}}

\newcommand{\uniformSymb}[0]{\mathcal{U}}
\newcommand{\uniform}[1]{\uniformSymb\left(#1\right)}
\newcommand{\setInsert}[0]{\xleftarrow{\scriptscriptstyle +}}
\newcommand{\setRemove}[0]{\xleftarrow{\scriptscriptstyle -}}
\newcommand{\lebesgue}[1]{\lebesgueSymb\left(#1\right)}

\newcommand{\stateSet}[0]{X}
\newcommand{\batchSet}[0]{Y}
\newcommand{\vertexSet}[0]{V}
\newcommand{\edgeSet}[0]{E}
\newcommand{\edgeCost}[0]{\cost}
\newcommand{\queue}[0]{\mathcal{Q}}

\newcommand{\treeGraph}[0]{\mathcal{T}}

\newcommand{\pathSet}[0]{\Sigma}
\newcommand{\bestPath}[0]{\pathSeq^{*}}

\newcommand{\pathCost}[1]{\pathCostSymb\left(#1\right)}

\newcommand{\samplePathBIT}[0]{\pathSeq_{\totalSamples,\mathrm{BIT}^*}}

\newcommand{\costFromBelow}[1]{\widehat{#1}}
\newcommand{\trueCost}[1]{#1}
\newcommand{\costFromAbove}[1]{#1_{\treeGraph}}
\newcommand{\gBelow}[1]{\costFromBelow{\costToCome}\left(#1\right)}
\newcommand{\gTrue}[1]{\trueCost{\costToCome}\left(#1\right)}
\newcommand{\gAbove}[1]{\costFromAbove{\costToCome}\left(#1\right)}
\newcommand{\hBelow}[1]{\costFromBelow{\costToGo}\left(#1\right)}

\newcommand{\fBelow}[1]{\costFromBelow{\solutionCost}\left(#1\right)}

\newcommand{\cBelow}[2]{\costFromBelow{\edgeCost}\left(#1,#2\right)}
\newcommand{\cTrue}[2]{\trueCost{\edgeCost}\left(#1,#2\right)}

\newcommand{\namedLebesgue}[1]{\lebesgueSymb_{#1}}
\newcommand{\phsMeasure}[0]{\namedLebesgue{\rm PHS}}
\newcommand{\sampleMeasure}[0]{\namedLebesgue{\rm sample}}
\newcommand{\namedSet}[1]{\stateSet_{#1}}

\newcommand{\fBelowSet}[0]{\namedSet{\costFromBelow{\solutionCost}}}

\newcommand{\obsSet}[0]{\namedSet{\rm obs}}
\newcommand{\freeSet}[0]{\namedSet{\rm free}}
\newcommand{\goalSet}[0]{\namedSet{\rm goal}}
\newcommand{\samplingSet}[0]{\namedSet{\rm sampling}}
\newcommand{\samplesSet}[0]{\namedSet{\rm samples}}
\newcommand{\unconnectedSet}[0]{\namedSet{\rm unconn}} %

\newcommand{\phsSet}[0]{\namedSet{\rm PHS}}

\newcommand{\nearSet}[0]{\namedSet{\rm near}}
\newcommand{\nearSetk}[0]{\namedSet{{\rm near},\counterk}}
\newcommand{\newSet}[0]{\namedSet{\rm new}}
\newcommand{\reuseSet}[0]{\namedSet{\rm reuse}}

\newcommand{\namedVertexSet}[1]{\vertexSet_{#1}}

\newcommand{\solnVertices}[0]{\namedVertexSet{\mbox{\rm\scriptsize sol'n}}}
\newcommand{\unexpandedVertices}[0]{\namedVertexSet{\rm unexpnd}} %
\newcommand{\delayedVertices}[0]{\namedVertexSet{\rm delayed}}
\newcommand{\nearVertices}[0]{\namedVertexSet{\rm near}}

\newcommand{\namedState}[1]{\statex_{#1}}
\newcommand{\xmin}[0]{\namedState{\rm min}}
\newcommand{\xstart}[0]{\namedState{\rm start}}
\newcommand{\xgoal}[0]{\namedState{\rm goal}}

\newcommand{\xrand}[0]{\namedState{\rm rand}}
\newcommand{\xnew}[0]{\namedState{\rm new}}
\newcommand{\namedVertex}[1]{\statev_{#1}}
\newcommand{\vgoal}[0]{\namedVertex{\rm goal}}

\newcommand{\vmin}[0]{\namedVertex{\rm min}}
\newcommand{\vnear}[0]{\namedVertex{\rm near}}
\newcommand{\vnearest}[0]{\namedVertex{\rm nearest}}
\newcommand{\vparent}[0]{\namedVertex{\rm parent}}

\newcommand{\namedCost}[1]{\cost_{#1}}

\newcommand{\ccur}[0]{\namedCost{\iter}}

\newcommand{\cnew}[0]{\namedCost{\rm new}}
\newcommand{\cmin}[0]{\namedCost{\rm min}}
\newcommand{\cnear}[0]{\namedCost{\rm near}}
\newcommand{\cedge}[0]{\namedCost{\rm edge}}
\newcommand{\creqd}[0]{\namedCost{\mbox{\rm\scriptsize req'd}}}
\newcommand{\csampled}[0]{\namedCost{\rm sampled}}

\newcommand{\namedQueue}[1]{\queue_{#1}}
\newcommand{\edgeQueue}[0]{\namedQueue{\edgeSet}}
\newcommand{\vertexQueue}[0]{\namedQueue{\vertexSet}}
\newcommand{\sampleQueue}[0]{\namedQueue{\rm Samples}}

\newcommand{\bestEdge}[0]{\pair{\vmin}{\xmin}}

\newcommand{\RRTsharp}{\texorpdfstring{\acs{RRT}\textsuperscript{\#}}{\acs{RRT}\#}}
\newcommand{\RRTx}{\texorpdfstring{\acs{RRT}\textsuperscript{X}}{\acs{RRT}x}}

\makeatletter
\def\@fnsymbol#1{\ensuremath{\ifcase#1\or \dagger\or \ddagger\or
   \mathsection\or \mathparagraph\or \|\or \dagger\dagger
   \or \ddagger\ddagger \else\@ctrerr\fi}}
\makeatother

\begin{document}
\begin{acronym}[Lazy Q-PRM]
    \acro{ASRL}{Autonomous Space Robotics Lab}
    \acro{CMU}{Carnegie Mellon University}
    \acro{CSA}{Canadian Space Agency}
    \acro{DRDC}{Defence Research and Development Canada}
    \acro{JPL}{Jet Propulsion Laboratory}
    \acro{KSR}{Koffler Scientific Reserve at Jokers Hill}
    \acro{MET}{Mars Emulation Terrain}
    \acro{MIT}{Massachusetts Institute of Technology}
    \acro{NASA}{National Aeronautics and Space Administration}
    \acro{NSERC}{Natural Sciences and Engineering Research Council of Canada}
    \acro{NCFRN}{\acs{NSERC} Canadian Field Robotics Network}
    \acro{NORCAT}{Northern Centre for Advanced Technology Inc.}
    \acro{ODG}{Ontario Drive and Gear Ltd.}
    \acro{ONR}{Office of Naval Research}
    \acro{USSR}{Union of Soviet Socialist Republics}
    \acro{UofT}{University of Toronto}
    \acro{UW}{University of Waterloo}
    \acro{UTIAS}{University of Toronto Institute for Aerospace Studies}

    \acro{ACPI}{advanced configuration and power interface}
    \acro{CLI}{command-line interface}
    \acro{GUI}{graphical user interface}
    \acro{JIT}{just-in-time}
    \acro{LAN}{local area network}
    \acro{MFC}{Microsoft foundation class}
    \acro{NIC}{network interface card}
    \acro{SDK}{software development kit}
    \acro{HDD}{hard-disk drive}
    \acro{SSD}{solid-state drive}

    \acro{ICRA}{IEEE International Conference on Robotics and Automation}
    \acro{IJRR}{International Journal of Robotics Research}
    \acro{IROS}{IEEE/RSJ International Conference on Intelligent Robots and Systems}
    \acro{RSS}{Robotics: Science and Systems Conference}
    \acro{TRO}{IEEE Transactions on Robotics}

    \acro{DOF}{degree-of-freedom}
        \acrodefplural{DOF}[DOFs]{degrees-of-freedom} %

        \acro{FOV}{field of view}
            \acrodefplural{FOV}[FOVs]{fields of view}
        \acro{HDOP}{horizontal dilution of position}
        \acro{UTM}{universal transverse mercator}
        \acro{WAAS}{wide area augmentation system}
        \acro{AHRS}{attitude heading reference system}
        \acro{DAQ}{data acquisition}
        \acro{DGPS}{differential global positioning system}
        \acro{DPDT}{double-pole, double-throw}
        \acro{DPST}{double-pole, single-throw}
        \acro{GPR}{ground penetrating radar}
        \acro{GPS}{global positioning system}
        \acro{LED}{light-emitting diode}
        \acro{IMU}{inertial measurement system}
        \acro{PTU}{pan-tilt unit}
        \acro{RTK}{real-time kinematic}
        \acro{R/C}{radio control}
        \acro{SCADA}{supervisory control and data acquisition}
        \acro{SPST}{single-pole, single-throw}
        \acro{SPDT}{single-pole, double-throw}
        \acro{UWB}{ultra-wide band}

    \acro{DDS}{Departmental Doctoral Seminar}
    \acro{DEC}{Doctoral Examination Committee}
    \acro{FOE}{Final Oral Exam}
    \acro{ICD}{Interface Control Document}

    \acro{iid}[i.i.d.]{independent and identically distributed}
    \acro{aas}[a.a.s.]{asymptotically almost-surely}
    \acro{asao}[a.s.a.o.]{almost-surely asymptotically optimal}
    \acro{RGG}{random geometric graph}
    \acro{BVP}{boundary-value problem}
    \acro{2BVP}[two-point BVP]{two-point \acl{BVP}}

    \acro{EKF}{extended Kalman filter}
    \acro{iSAM}{incremental smoothing and mapping}
    \acro{ISRU}{in-situ resource utilization}
    \acro{PCA}{principle component analysis}
    \acro{SLAM}{simultaneous localization and mapping}
    \acro{SVD}{singular value decomposition}
    \acro{UKF}{unscented Kalman filter}
    \acro{VO}{visual odometry}
    \acro{VTR}[VT\&R]{visual teach and repeat}

        \acro{ADstar}[AD*]{Anytime Dynamic A*}
        \acro{ARAstar}[ARA*]{Anytime Repairing A*}
        \acro{BI2RRTstar}[BI\textsuperscript{2}RRT*]{Bidirectional Informed \acs{RRTstar}}
        \acro{BITstar}[BIT*]{Batch Informed Trees}
            \acrodefplural{BITstar}[BIT*]{Batch Informed Trees}
        \acro{RABITstar}[RABIT*]{Regionally Accelerated \acs{BITstar}}
        \acro{BRM}{belief roadmap}
        \acro{CFOREST}[C-FOREST]{Coupled Forest of Random Engrafting Search Trees}
        \acro{CHOMP}{Covariant Hamiltonian Optimization for Motion Planning}
        \acro{Dstar}[D*]{Dynamic A*}
        \acro{EET}{Exploring/Exploiting Tree} %
        \acro{EST}{Expansive Space Tree}
            \acrodefplural{EST}[EST]{Expansive Space Trees}
        \acro{FMTstar}[FMT*]{Fast Marching Tree}
            \acrodefplural{FMTstar}[FMT*]{Fast Marching Trees}
        \acro{GVG}{generalized Voronoi graph}
        \acro{HRAstar}[HRA*]{Hybrid Randomized A*}
        \acro{hRRT}{Heuristically Guided \acs{RRT}}
        \acro{LBTRRT}[LBT-RRT]{Lower Bound Tree \acs{RRT}}
        \acro{LQG-MP}{linear-quadratic Gaussian motion planning}
        \acro{LPAstar}[LPA*]{Lifelong Planning A*}
        \acro{MHAstar}[MHA*]{Multi-Heuristic A*}
        \acro{MPLB}{Motion Planning Using Lower Bounds}
        \acro{MDP}{Markov decision process}
            \acrodefplural{MDP}[MDPs]{Markov decision processes}
        \acro{NRP}{network of reusable paths}
            \acrodefplural{NRP}[NRPs]{networks of reusable paths}
        \acro{POMDP}{partially observable \ac{MDP}}
            \acrodefplural{POMDP}[POMDPs]{partially observable \acp{MDP}}
        \acro{PRM}{Probabilistic Roadmap}
        \acro{PRMstar}[PRM*]{almost-surely asymptotically optimal \acs{PRM}}
            \acrodefplural{PRMstar}[PRM*]{almost-surely asymptotically optimal \acsp{PRM}}
        \acro{QPRM}[Q-PRM]{Quasi-random Roadmap}
            \acro{QLPRM}[QLPRM]{Quasi-random Lazy \acs{PRM}}
        \acro{RAstar}[RA*]{Randomized A*}
        \acro{RRBT}{rapidly-exploring random belief tree}
        \acro{RRG}{Rapidly exploring Random Graph}
        \acro{RRM}{Rapidly exploring Roadmap}
        \acro{RRT}{Rapidly exploring Random Tree}
        \acro{RRTstar}[RRT*]{almost-surely asymptotically optimal \acs{RRT}}
            \acrodefplural{RRTstar}[RRT*]{almost-surely asymptotically optimal \acsp{RRT}}
        \acro{SANDROS}{selective and non-uniformly delayed refinement of subgoals}
        \acro{SBAstar}[SBA*]{Sampling-based A*}
        \acro{SBL}{Single-query, Bidirectional, Lazy Collision-checking Planner}
        \acro{SORRTstar}[SORRT*]{Sorted \acs{RRTstar}}
        \acro{sPRM}[s-PRM]{simplified \acs{PRM}}
            \acrodefplural{sPRM}[s-PRM]{simplified \acsp{PRM}}
        \acro{SRT}{Sampling-based Roadmap of Trees}
            \acrodefplural{SRT}{Sampling-based Roadmaps of Trees}
        \acro{STOMP}{Stochastic Trajectory Optimization for Motion Planning}
        \acro{TRRT}[T-RRT]{Transition-based \acs{RRT}}
            \acro{TRRTstar}[T-RRT*]{Transition-based \acs{RRTstar}}
            \acro{ATRRT}[AT-RRT]{Anytime Transition-based \acs{RRT}}
        \acro{TrajOpt}{constrained trajectory optimization}

        \acro{ATDstar}[ATD*]{Anytime Truncated \acs{Dstar}}
        \acro{TDstarLite}[TD* Lite]{Truncated \acs{Dstar} Lite} 
        \acro{TLPAstar}[TLPA*]{Truncated \acs{LPAstar}}

    \acro{HERB}{Home Exploring Robot Butler}
    \acro{MER}{Mars Exploration Rover}
        \acrodefplural{MER}[MER]{Mars Exploration Rovers}
    \acro{MSL}{Mars Science Laboratory}
    \acro{OMPL}{Open Motion Planning Library}
    \acro{ROS}{the Robot Operating System}
    \acro{UAV}{unmanned aerial vehicle}

\end{acronym} %

\runninghead{\UTIASauthor}

\title{\UTIAStitle}

\author{%
Jonathan D.\ Gammell\affilnum{1}, %
Timothy D.\ Barfoot\affilnum{2}, and %
Siddhartha S.\ Srinivasa\affilnum{3}%
}

\affiliation{%
\affilnum{1}University of Oxford, Oxford, United Kingdom. Work performed while at the University of Toronto, Toronto, Canada\\%
\affilnum{2}University of Toronto, Toronto, Canada\\%
\affilnum{3}University of Washington, Seattle, USA%
}

\corrauth{Jonathan D.\ Gammell,
University of Oxford,
Oxford,
OX2~6NN, UK.}

\email{gammell@robots.ox.ac.uk}

\begin{abstract}
    \acused{DOF}
    \acused{RRTstar}
    \acused{PRMstar}

Path planning in robotics often requires finding high-quality solutions to continuously valued and/or high-dimensional problems.
These problems are challenging and most planning algorithms instead solve simplified approximations.
Popular approximations include graphs and random samples, as respectively used by informed graph-based searches and anytime sampling-based planners.

Informed graph-based searches, such as A*, traditionally use heuristics to search \textit{a priori} graphs in order of potential solution quality.
This makes their search efficient but leaves their performance dependent on the chosen approximation.
If its resolution is too low then they may not find a (suitable) solution but if it is too high then they may take a prohibitively long time to do so.

Anytime sampling-based planners, such as \ac{RRTstar}, traditionally use random sampling to approximate the problem domain incrementally.
This allows them to increase resolution until a suitable solution is found but makes their search dependent on the order of approximation.
Arbitrary sequences of random samples approximate the problem domain in every direction simultaneously and but may be prohibitively inefficient at containing a solution.

This paper unifies and extends these two approaches to develop \ac{BITstar}, an informed, anytime sampling-based planner.
\ac{BITstar} solves continuous path planning problems efficiently by using sampling and heuristics to alternately approximate and search the problem domain.
Its search is ordered by potential solution quality, as in A*, and its approximation improves indefinitely with additional computational time, as in \acs{RRTstar}.
It is shown analytically to be almost-surely asymptotically optimal and experimentally to outperform existing sampling-based planners, especially on high-dimensional planning problems.
\end{abstract}


\maketitle

\acresetall %
\acused{DOF}
\acused{RRTstar}
\acused{PRMstar}

\begin{figure*}[tbp]
    \centering
    \includegraphics[page=2,width=\textwidth]{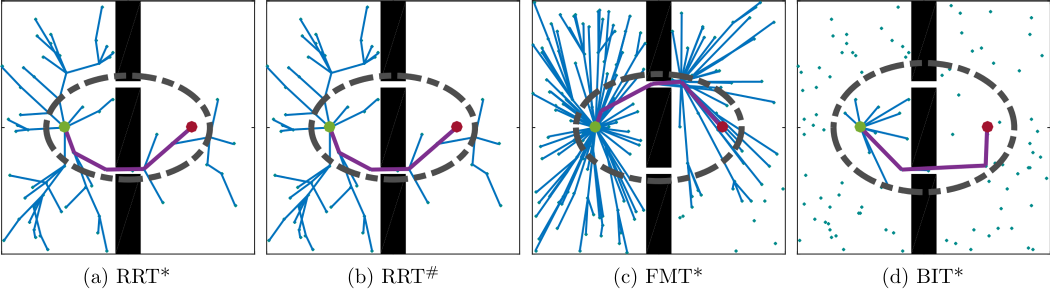}
    \caption
    {%
        \figureCaptionStyle%
        The relative \emph{efficiency} of the searches performed by \acs{RRTstar} (a), \RRTsharp{} (b), \acs{FMTstar} (c), and \acs{BITstar} (d) as illustrated by their initial solution and its resulting informed set.
        Each planner was run until an initial solution was found and then stopped.
        The resulting set of states that could provide a better solution is shown as a dark-grey dashed line.
        Search outside this informed set is provably unnecessary for the solution found by each planner and illustrates the inefficiency of the initial search of \acs{RRTstar}, \RRTsharp{}, and \acs{FMTstar}, (a)--(c).
        Note that \RRTsharp{} finds an initial solution from the same samples as \acs{RRTstar} as the heuristics presented in \citet{arslan_icra15} do not alter the search until a solution is found.
        Also note that by ordering its search on potential solution quality, \acs{BITstar} does not consider any samples that cannot provide the best solution in its current approximation (i.e., batch of samples).
        The start and goal states of the problem are shown in green and red, respectively, while the graph built by the planner is shown in blue and the initial solution is highlighted in purple.
    }
    \label{fig:order}
\end{figure*}%

\section{Introduction}\label{sec:intro}
A fundamental task of any robotic system operating autonomously in dynamic and/or unknown environments is the ability to navigate between positions.
The difficulty of this path planning problem depends on the number of possible robot positions (i.e., \emph{states}) that could belong to a solution (i.e., the size of the \emph{search space}).
This search space is often continuous since robots can be infinitesimally repositioned and also often unbounded (e.g., planning outdoors) and/or high dimensional (e.g., planning for manipulation).

Most global path planning algorithms reduce this search space by considering a countable subset of the possible states (i.e., a \emph{discrete approximation}).
This simplifies the problem but limits formal algorithm performance to the chosen discretization.
Popular discretizations in robotics include \textit{a priori} graph-{} and anytime sampling-based approximations.

\textit{A priori} graph-based approximations can be searched efficiently with informed algorithms, such as A* \citep{hart_tssc68}.
These informed graph-based searches not only find the optimal solution to a representation (i.e., they are \emph{resolution optimal}) but do so efficiently. %
For a chosen heuristic, A* is the \emph{optimally efficient} search of a given graph \citep{hart_tssc68}.

This efficient search makes informed graph-based algorithms effective on many continuous path planning problems despite a dependence on the chosen approximation.
Sparse approximations can be searched quickly but may only contain low quality \emph{continuous} solutions (if they contain any).
Dense approximations alternatively contain high-quality continuous solutions \citep{bertsekas_tac75} but may be prohibitively expensive to search.
Choosing the correct resolution \textit{a priori} to the search is difficult and is exacerbated by the exponential growth of graph size with state dimension \citeParenthetical{i.e., the \textit{curse of dimensionality}}{bellman_ams54,bellman_57}.

Anytime sampling-based approximations are instead built and searched simultaneously by sampling-based algorithms such as \citeAcronymDefnP{PRM}{kavraki_tro96}, \citeAcronymDefnP{RRT}{lavalle_ijrr01}, and \ac{RRTstar} \citep{karaman_ijrr11}.
These sampling-based planners incrementally improve their approximation of the problem domain until a (suitable) solution is found (i.e., they have \textit{anytime resolution}) and have probabilistic performance.
The probability that they find a solution, if one exists, goes to unity with an infinite number of samples (i.e., they are \emph{probabilistically complete}).
The probability that algorithms such as \ac{RRTstar} and versions of \ac{PRM} \citeParentheticalTwice{e.g., \acs{sPRM}}{kavraki_tra98}{\ac{PRMstar}}{karaman_ijrr11} asymptotically converge towards the optimum, if one exists, also goes to unity with an infinite number of samples \citeParenthetical{i.e., they are \emph{almost-surely asymptotically optimal}}{karaman_ijrr11}.
\squeezeWidowedWords

\begin{figure*}[tbp]
    \centering
    \includegraphics[width=\textwidth]{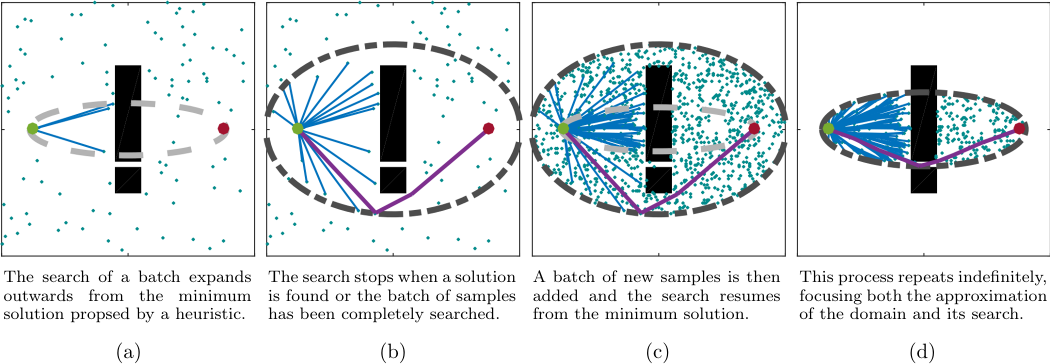}
    \caption[A simple example of the ordered sampling-based search performed by \acs{BITstar}.]
    {%
        \figureCaptionStyle%
        A simple example of the ordered sampling-based search performed by \acs{BITstar}.
        The start and goal states are shown in green and red, respectively, while the current solution is highlighted in purple.
        The set of states that can provide a better solution (the \emph{informed} set) is shown as a dark grey dashed line, while the progress of the current batch is shown as a light grey dashed line illustrating the informed set that the current edge would define.
        Fig.~(a) shows the growing search of the first batch of samples, and (b) shows the first search ending when a solution is found.
        Note that ordering the search finds a solution without considering any states outside the informed set it defines.
        Fig.~(c) shows the search continuing after pruning the representation and increasing its accuracy with a new batch of samples while (d) shows this second search ending when an improved solution is found.
    }
    \label{fig:bitstar}
\end{figure*}%

This anytime representation makes sampling-based planners effective on many continuous path planning problems despite a dependence on sampling.
Incremental planners, such as \acs{RRTstar}, generally search the problem domain in the order given by their random samples (Fig.~\ref{fig:order}).
This simultaneously expands the search into the entire sampling domain (i.e., it is \emph{space filling}) and almost-surely converges asymptotically to the optimal solution by asymptotically finding optimal paths to \emph{every} state.
This random order is inefficient and wastes computational effort on states that are not used to find a solution.
This may be prohibitively expensive in many real problems, including unbounded or high-dimensional environments.

Previous work combining the complementary advantages of these two approaches has been incomplete.
Sampling concepts have been added to informed graph-based search by either sacrificing anytime resolution or by decreasing efficiency by ordering the search on metrics other than solution cost.
Heuristic concepts have been added to sampling-based planners by either also sacrificing anytime resolution or by decreasing efficiency by only applying heuristics to some aspects of the search.

This paper demonstrates how informed graph-based search and sampling-based planning can be directly unified and extended without compromise.
It presents \ac{BITstar} as an example of an informed, anytime sampling-based planner that incrementally approximates continuous planning problems while searching in order of potential solution quality (Fig.~\ref{fig:bitstar}).
This efficient approach avoids unnecessary computational costs while still almost-surely converging asymptotically to the optimum and finding better solutions faster than existing algorithms, especially in high state dimensions.

\ac{BITstar} approximates continuous search spaces with an \textit{edge-implicit} \citeAcronymDefn{RGG}{penrose_03}.
This \ac{RGG} is defined by a set of random samples and an appropriate connection condition.
The accuracy of this approximation improves as the number of samples increases and almost-surely converges towards containing the optimal solution as the number of samples approaches infinity, similar to \ac{RRTstar}.

This improving approximation is maintained and searched using heuristics.
The initial approximation is searched in order of potential solution quality, as in A*.
When it is improved the search is updated efficiently by reusing previous information, as in incremental search techniques such as \citeAcronymDefn{LPAstar}{koenig_ai04} and \citeAcronymDefn{TLPAstar}{aine_ai16}.
The improving approximation is focused to the \emph{informed set} of states that could provide a better solution, as in Informed \acs{RRTstar} \citep{gammell_iros14,gammell_tro18}.
\squeezeWidowedWords

\ac{BITstar} only requires three user-defined options,
\begin{inparaenum}[(i)]%
    \item a \ac{RGG}-connection parameter,
    \item the heuristic function, and
    \item the number of samples per batch
\end{inparaenum}%
and can generalize existing algorithms (Fig.~\ref{fig:taxonomy}).
With a single batch of samples and no heuristic, it is a search of a static approximation ordered by cost-to-come.
This is exactly Dijkstra's algorithm \citep{dijkstra_59} on a \ac{RGG} or a version of \citeAcronymDefn{FMTstar}{janson_ijrr15}.
With a single batch of samples and a heuristic, it is a search of a static approximation ordered by estimated solution quality.
This is exactly a lazy version of A* \citeExample{cohen_rss14} on a \ac{RGG}.
With multiple batches and a heuristic, it is a truncated incremental search of a changing approximation ordered by estimated solution quality.
This is equivalent to a lazy version of \ac{TLPAstar} on a \ac{RGG} where replanning is always truncated after one propagation.
With multiple batches of one sample and no heuristic, it is a construction of a tree through incremental sampling.
This is equivalent to a `\texttt{steer}'-free version of \ac{RRTstar} where unsuccessful samples are maintained.

\begin{figure}[tb]
    \centering
    \includegraphics[page=1,scale=1]{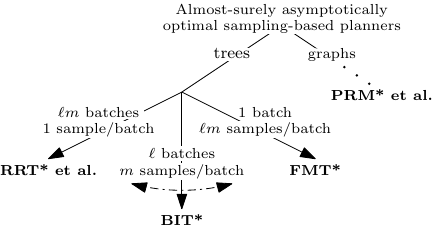}
    \caption
    {%
        \figureCaptionStyle%
        A simplified taxonomy of almost-surely asymptotically optimal sampling-based planners that demonstrates the relationship between \acs{RRTstar}, \acs{FMTstar}, and \acs{BITstar}.
        When using a batch size of a single sample, \acs{BITstar} is a version of \acs{RRTstar}.
        When using a single batch consisting of multiple-samples, \acs{BITstar} is a version of \acs{FMTstar}.
    }
    \label{fig:taxonomy}
\end{figure}%

\ac{BITstar} can also be extended to create new planners.
\ac{SORRTstar} is presented as an extension of heuristically ordered search concepts to the incremental sampling of \acs{RRTstar} algorithms.
A version of both \ac{BITstar} and \ac{SORRTstar} are publicly available in the \citeAcronymDefn{OMPL}{ompl}.

The benefits of performing an ordered anytime search of the problem domain are demonstrated on both abstract problems and experiments for the \acs{CMU} Personal Robotic Lab's \ac{HERB}, a 14-\ac{DOF} mobile manipulation platform \citep{herb}.
The results show that \ac{BITstar} finds better solutions faster than existing almost-surely asymptotically optimal planners and also \ac{RRT}, especially in high dimensions.
The only tested planner that found (worse) solutions faster was \acs{RRT}-Connect \citep{kuffner_icra00}, a bidirectional version of \ac{RRT} that is \emph{not} almost-surely asymptotically optimal.
\squeezeWidowedWords

The remainder of this paper is organized as follows.
Section~\ref{sec:lit} presents a review of previous work combining aspects of informed search and sampling-based planning.
Section~\ref{sec:bitstar} presents \ac{BITstar} as an efficient search of an increasingly dense approximation of a continuous planning problem. %
Section~\ref{sec:anal} proves that \ac{BITstar} is probabilistic complete and almost-surely asymptotically optimal and demonstrates its relationship to \ac{LPAstar} and \ac{TLPAstar}.
Section~\ref{sec:mods} presents simple extensions of \ac{BITstar} that may further improve performance in some planning situations, including \ac{SORRTstar}.
Section~\ref{sec:exp} demonstrates the benefits of \ac{BITstar} on both abstract problems and experiments for \ac{HERB}.
Section~\ref{sec:fin} finally presents a closing discussion and thoughts on future work.
\squeezeWidowedWords

\subsection{Relationship to Previous Publications}\label{sec:intro:prev}
This paper is distinct from our prior work analyzing the necessary conditions for anytime sampling-based planners to improve an existing solution \citeParenthetical{``Informed sampling for asymptotically optimal path planning''}{gammell_tro18}.
This other publication investigates focusing the search for improvements in anytime sampling-based planners.
It presents both analytical concepts that are applicable to any cost function (e.g., informed sets) and algorithmic approaches specific to problems seeking to minimize path length (e.g., direct informed sampling of the $L^2$ informed set).

In comparison, this paper presents algorithmic approaches to order the search for both an initial solution and subsequent improvements.
These techniques are applicable to any cost function, similarly to A*, and make use of the appropriate analytical results from \citet{gammell_tro18}.
The performance of these techniques are illustrated on a number of problems seeking to minimize path length (Section~\ref{sec:exp}) where we make use of the appropriate algorithmic results of \citet{gammell_tro18}.
\squeezeWidowedWords

\subsection{Statement of Contributions}\label{sec:intro:claims}
This paper is a continuation of ideas that were first presented at an \acs{RSS} workshop \citep{gammell_igmpw14} along with a supporting technical report \citep{gammell_arxiv14c} and were then published in \citet{gammell_icra15} and \citet{gammell_phd17}.
It makes the following specific contributions:
\begin{itemize}
    \item Presents an updated version of \ac{BITstar} that avoids repeated consideration of failed edges and is a direct unification of informed graph-based search and sampling-based planning (\algos{algo:bitstar}{algo:prune}).
    \item Develops new extensions to \ac{BITstar}, including \ac{JIT} sampling for the direct search of unbounded problems (\algo{algo:jit}) and \ac{SORRTstar} (\algo{algo:sorrtstar}) as an application of ordered search concepts to the algorithmic simplicity of \ac{RRTstar}.
    \item Provides expanded proof that \ac{BITstar} is probabilistically complete and almost-surely asymptotically optimal (Theorems~\ref{thm:pc} and \ref{thm:asao}) and has a queue ordering equivalent to a lazy \ac{TLPAstar} (Lemma~\ref{lem:lpa}).
    \item Demonstrates experimentally the benefits of using ordered search concepts to combat the curse of dimensionality.
\end{itemize}

\section{Prior Work Ordering Sampling-based Planners}\label{sec:lit}
A formal definition of the optimal path planning problem is presented to motivate a review of existing literature on the combination of informed graph-based and sampling-based concepts (Section~\ref{sec:lit:opt}).
This prior work can be loosely classified as either incorporating sampling into informed A*-style searches (Section~\ref{sec:lit:astar}) or adding heuristics to incremental \acs{RRT}/\acs{RRTstar}-style searches (Section~\ref{sec:lit:rrt}).

\subsection{The Optimal Path Planning Problem}\label{sec:lit:opt}
The two most common path planning problems in robotics are those of \emph{feasible} and \emph{optimal} planning.
Feasible path planning seeks a path that connects a start to a goal in a search space while avoiding obstacles and obeying the differential constraints of the robot.
Optimal path planning seeks the feasible path that minimizes a chosen cost function (e.g., path length).
By definition, solving an optimal planning problem requires solving the underlying feasible problem.

The optimal path planning problem is defined in Definition~\ref{defn:back:opt} similarly to \citet{karaman_ijrr11}.
This definition is expressed in the state space of a robot but can be posed in other representations, including configuration space \citep{lozano-perez_tc83}.
The goal of these problems may be a single goal state (e.g., a pose for a mobile robot) or any state in a goal \textit{region} (e.g., the set of joint angles that give a redundant manipulator a desired end-effector position).

\begin{defn}[The optimal path planning problem]\label{defn:back:opt}
    Let $\stateSet \subseteq \real^\dimension$ be the state space of the planning problem, $\obsSet \subset \stateSet$ be the states in collision with obstacles, and $\freeSet = \closure{\stateSet \setminus \obsSet}$ be the resulting set of permissible states, where $\closure{\cdot}$ represents the closure of a set.
    Let $\xstart \in \freeSet$ be the initial state and $\goalSet \subset \freeSet$ be the set of desired goal states.
    Let $\pathSeq : \; \left[0,1\right] \to \freeSet$ be a continuous map to a sequence of states through collision-free space of bounded variation that can be executed by the robot (i.e., a collision-free, feasible path) and $\pathSet$ be the set of all such nontrivial paths.
    \squeezeWidowedWords

    The optimal path planning problem is then formally defined as the search for a path, $\bestPath\in\pathSet$, that minimizes a given cost function, $\pathCostSymb : \; \pathSet \to \positiveReal$, while connecting $\xstart$ to $\xgoal\in\goalSet$,
    \begin{align*}
        \bestPath = \argmin_{\pathSeq \in \pathSet} \setst{\pathCost{\pathSeq}}{\pathSeq\left(0\right) = \xstart,\, \pathSeq\left(1\right) \in \goalSet},
    \end{align*}%
    where $\positiveReal$ is the set of non-negative real numbers.
\end{defn}

\subsection{A*-based Approaches}\label{sec:lit:astar}
A* is a popular planning algorithm because it is the optimally efficient search of a graph.
Any other algorithm guaranteed to find the resolution-optimal solution will expand at least as many vertices when using the same heuristic \citep{hart_tssc68}.
This efficiency is achieved by ordering the search by potential solution quality such that the optimal solution is found by only considering states that could have provided a better solution. %
Applying A* to continuous path planning problems requires discretizing the search domain and significant work has incorporated sampling into this process, often to avoid the need to do so \textit{a priori}.

\citet{sallaberger_acta95} demonstrate the advantages of including stochasticity in the representation of a continuous state space by adding random perturbations to a regular discretization.
They solve path planning problems for spacecraft and multilink robotic arms using dynamic programming and show that their technique finds better solutions than regular discretizations alone; however, the approximation is still defined \textit{a priori} to the search.

\citeAcronymDefn{RAstar}{diankov_iros07} uses random sampling to apply A* directly to continuous planning problems.
Vertices are expanded a user-specified number of times by randomly sampling possible nearby descendants.
If a sample is sufficiently different than the existing states in the tree and can be reached without collision it is added as a child of the expanded vertex.
The resulting sampling-based search expands outwards from the start in order of potential solution quality but is not anytime.
If a (suitable) solution is not found then the search must be restarted with a different number of samples per vertex.

\citeAcronymDefn{HRAstar}{teniente_iros13} performs a similar search for systems with differential constraints by sampling control inputs instead of states.
Vertices are expanded a user-specified number of times by generating random control inputs and propagating them forward in time with a motion model.
These new states are used both to expand a tree and rewire it when they are sufficiently similar to existing states.
The vertex to expand is selected with a hybrid cost policy that considers path cost, the number of nearby vertices, and the distance to obstacles.
This biases sampling into unexplored regions of the problem domain but may waste computational effort on states that are unnecessary to find the eventual solution.

\citeAcronymDefn{SBAstar}{persson_ijrr14} performs an ordered search by \emph{iteratively} expanding vertices in a heuristically weighted version of \citeAcronymDefn{EST}{hsu_ijcga99}.
At each iteration, a \emph{single} random sample is generated near the vertex at the front of a priority queue.
This queue is ordered on the potential solution quality of vertices and the likelihood of sampling a \emph{unique} and \emph{useful} sample in their neighbourhoods.
These likelihoods are required to bias sampling into unexplored space since vertices are never removed from the queue.
They are estimated from previous collision checks and estimates of the sample density around vertices.
Again, this search order on metrics other than the optimization objective may waste computational effort on states that are unnecessary to find the eventual solution.

\escapeParagraph
Unlike these methods to incorporate sampling into informed graph-based search, \ac{BITstar} maintains anytime performance while ordering its search only on potential solution quality.
This allows it to be run indefinitely until a suitable solution is found, return incrementally better solutions, and almost-surely asymptotically converge to the optimum.
This also limits its search of each representation to the set of states that were believed to be capable of providing a better solution than the one eventually found.

\subsection{\acs{RRT}-based Approaches}\label{sec:lit:rrt}
\ac{RRT} and \ac{RRTstar} solve continuous path planning problems by using incremental sampling to build a tree through obstacle-free space.
This avoids \textit{a priori} discretizations of the problem domain and allows them to be run indefinitely until a (suitable) solution is found.
This also makes the search dependent on the sequence of samples (i.e., makes it random) and significant work has sought ways to use heuristics to order the search.
Heuristics can also be used to focus the search, as in Anytime \acsp{RRT} \citep{ferguson_iros06} and Informed \acs{RRTstar} \citep{gammell_iros14,gammell_tro18} but this does not change the order of the search.
\squeezeWidowedWords

\citeAcronymDefn{hRRT}{urmson_iros03} probabilistically includes heuristics in the \ac{RRT} expansion step.
Randomly sampled states are probabilistically added in proportion to their heuristic value relative to the tree.
This balances the Voronoi bias of the \ac{RRT} search with a preference towards expanding potentially high-quality paths.
This improves performance, especially in problems with continuous cost functions \citeParenthetical{e.g., path length}{urmson_iros03}, but maintains a nonzero probability of wasting computational effort on states that are unnecessary for the eventual solution.

\citet{karaman_icra11} and \citet{akgun_iros11} both use heuristics to accelerate the convergence of \ac{RRTstar}.
\citet{karaman_icra11} remove vertices whose \emph{current} cost-to-come plus a heuristic estimate of cost-to-go is higher than the current solution.
\citet{akgun_iros11} reject samples that are heuristically estimated to be unable to provide a better solution.
These approaches both avoid unnecessary computational effort but do not alter the initial search as heuristics are not applied until a solution is found.

\citet{kiesel_socs12} use a two-stage process to order sampling for \ac{RRTstar} in their \emph{f-biasing} technique.
A heuristic is first calculated by solving a coarse discretization of the planning problem with Dijkstra's algorithm.
This heuristic is then used to bias \ac{RRTstar} sampling to the regions of the problem domain where solutions were found.
This increases the likelihood of finding solutions quickly but maintains a nonzero sampling probability over the entire problem domain for the whole search and allows search effort to be wasted on states unnecessary for the eventual solution.

\RRTsharp{} \citep{arslan_icra13,arslan_icra15} uses heuristics to find and update a tree in the graph built incrementally by \citeAcronymDefn{RRG}{karaman_ijrr11}.
It efficiently maintains the optimal connection to each vertex by using \ac{LPAstar} to propagate changes through the \emph{entire} graph.
This can also be implemented as policy iteration \citep{arslan_cdc16}.

\RRTsharp{} uses heuristics to limit this graph to regions of the problem domain that can improve an \emph{existing} solution.
This focuses the search for an improvement but does not alter the \emph{order} in which the graph itself is constructed.
As in \ac{RRTstar}, this graph is built in the order given by the sampling sequence using \ac{RRT}-style incremental techniques.
This may cause important, difficult-to-sample states to be discarded simply because they cannot currently be connected to the tree.
It may also waste computational effort, especially during the search for an initial solution (Fig.~\ref{fig:order}).

\RRTx{} \citep{otte_wafr14,otte_ijrr16} extends propagated rewiring to dynamic environments by limiting propagation to changes that improve the graph by more than a user-specified threshold.
This \emph{$\epsilon$-consistency} balances the local rewiring performed by \ac{RRTstar} and the cascaded rewiring performed by \RRTsharp{} to improve performance.
As in \ac{RRTstar} and \RRTsharp{}, it does not order the construction of the graph itself.

\ac{FMTstar} \citep{janson_isrr13,janson_ijrr15} uses a marching method to search a batch of samples in order of increasing cost-to-come, similar to Dijkstra's algorithm.
This separation of approximation and search makes the search order independent of the sampling order but sacrifices anytime performance.
Solutions are not returned until the search is finished and the search must be restarted if a (suitable) solution is not found, as with any other \textit{a priori} discretization.

The \acf{MPLB} algorithm \citep{salzman_icra15} extends \ac{FMTstar} with a heuristic and quasi-anytime resolution.
Quasi-anytime performance is achieved by \emph{independently} searching increasingly large batches of samples and returning the improved solutions when each individual search finishes.
It is stated that this can be done efficiently by reusing information but no specific methods are presented.

\escapeParagraph
Unlike these methods to incorporate informed search concepts into single-query sampling-based algorithms, \ac{BITstar} orders its entire search only on potential solution quality and still maintains anytime performance.
This avoids wasting computational effort by only considering states that are believed to be capable of providing the best solution given the current representation.
This also returns suboptimal solutions in an anytime manner while almost-surely converging asymptotically to the optimum.

\section[\acf{BITstar}]{\acf{BITstar}\footnotemark{}}\label{sec:bitstar}
\footnotetext{Pronounced \textit{bit star}.}%
Any discrete set of states distributed in the state space of a planning problem, $\samplesSet \subset \stateSet$, can be viewed as a graph whose edges are given algorithmically by a transition function (an \emph{edge-implicit} graph).
When these states are sampled randomly, $\samplesSet = \set{\statex \sim \uniform{\stateSet}}$, this graph is known as a \ac{RGG} and has studied probabilistic properties \citep{penrose_03}.

The connections (edges) between states (vertices) in a \ac{RGG} depend on their relative geometric position.
Common \acp{RGG} have edges to a specific number of each vertex's nearest neighbours \citeParenthetical{a $k$-nearest graph}{eppstein_dcg97} or to all neighbours within a specific distance \citeParenthetical{an $r$-disc graph}{gilbert_siam61}.
\ac{RGG} theory provides probabilistic relationships between the number and distribution of vertices, the $k$ or $r$ defining the graph, and specific properties such as connectivity or relative cost through the graph \citep{penrose_03,muthukrishnan_siam05}.

\begin{figure*}[tbp]
    \centering
    \includegraphics[page=1,width=\textwidth]{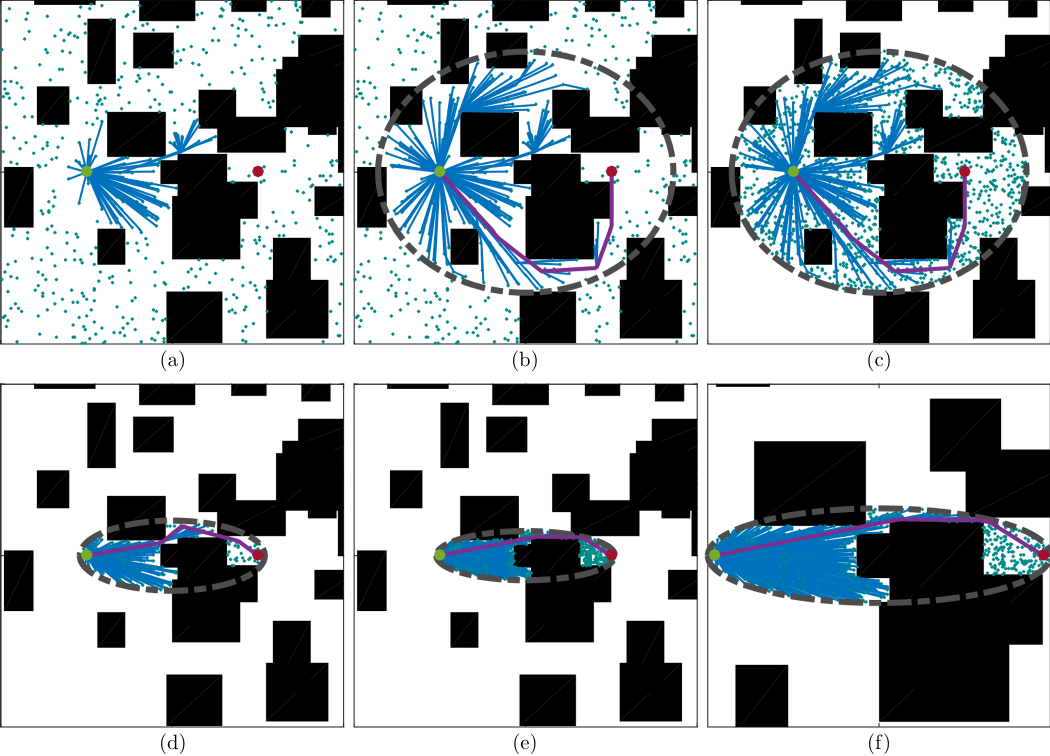}
    \caption[An example of how \acs{BITstar} uses incremental techniques to efficiently search batches of samples in order of potential solution quality.]
    {%
        \figureCaptionStyle%
        An example of how \acs{BITstar} uses incremental techniques to efficiently search batches of samples in order of potential solution quality.
        The search starts with a batch of samples uniformly distributed in the planning domain.
        This batch is searched outward from the start in order of potential solution quality, (a).
        The search continues until the batch is exhausted or a solution is found that cannot be improved with the current samples, (b).
        New samples are then added to the informed set and incremental techniques are used to continue the search, (c)--(e).
        This results in an algorithm that performs an ordered anytime search that almost-surely asymptotically converges to the optimal solution, shown enlarged in (f).
        Note that \acs{BITstar} orders all aspects of the search and never considers states in a batch that cannot provide the best solution (i.e., searches only inside the informed set defined by the eventual solution in the current graph).
    }
    \label{fig:example}
\end{figure*}

Anytime sampling-based planners can be viewed as algorithms that construct a graph in the problem domain and search it.
Some algorithms perform this construction and search simultaneously (e.g., \ac{RRTstar}) and others separately (e.g., \ac{FMTstar}) but, as in graph-based search, their performance always depends on both the accuracy of their approximation and the quality of their search.
\ac{RRTstar} uses \ac{RGG} theory to limit graph complexity while maintaining probabilistic bounds on approximation accuracy but incompletely searches the graph in the order it is constructed (i.e., performs a \emph{random} search).
\RRTsharp{} and \RRTx{} exploit the constructed graphs more thoroughly than \ac{RRTstar} but do not alter the order of its construction (i.e., they depend on the same \emph{random} search of the underlying problem domain).
\ac{FMTstar} performs a complete ordered search but uses \ac{RGG} theory to define an \textit{a priori} approximation of the problem domain (i.e., it is not anytime).
\squeezeWidowedWords

\ac{BITstar} uses \ac{RGG} theory to limit graph complexity while simultaneously building the graph in an anytime manner and searching it in order of potential solution quality.
This is made possible by using \emph{batches} of random samples to build an increasingly dense \textit{edge-implicit} \ac{RGG} in the informed set and using incremental search techniques to update the search (Fig.~\ref{fig:example}).
The anytime approximation allows \ac{BITstar} to run indefinitely until a (suitable) solution is found.
The ordered search avoids unnecessary computational cost by only considering states when they represent the best possible solution to a given approximation.
The incremental search allows it to update its changing representation efficiently by reusing previous information.

The complete \ac{BITstar} algorithm is presented in \algos{algo:bitstar}{algo:prune} and Sections~\ref{sec:bitstar:note}--\ref{sec:bitstar:algo:prune} with a discussion on some practical considerations presented in Section~\ref{sec:bitstar:algo:pract}.
For simplicity, the discussion is limited to a search from a single start state to a finite set of goal states with a constant batch size.
This formulation can be directly extended to searches from a start or goal region, and/or with variable batch sizes.
This version is publicly available in \ac{OMPL}.

\ac{BITstar} builds an explicit spanning tree of the implicit \ac{RGG} defined by a batch of samples.
The graph initially consists of only the start and goal (\algolines{algo:bitstar}{startInit}{endInit}; Section~\ref{sec:bitstar:algo:init}) but is incrementally grown with batches of new samples during the search (\algolines{algo:bitstar}{startBatch}{endBatch}; Section~\ref{sec:bitstar:algo:batch}).
The graph is searched in order of potential solution quality by selecting the best possible edge from a queue ordered by potential solution cost (\algolines{algo:bitstar}{startNextEdge}{endNextEdge}; Section~\ref{sec:bitstar:algo:getEdge}) and considering whether this edge improves the cost-to-come of its target vertex and could improve the current solution (\algolines{algo:bitstar}{startProcess}{endProcess}; Section~\ref{sec:bitstar:algo:addEdge}).
The search continues until no edges in the implicit \ac{RGG} could provide a better solution, at which point the accuracy of the approximation is increased by adding a batch of new samples and the search is resumed.
This process continues indefinitely until a suitable solution is found.

\subsection{Notation}\label{sec:bitstar:note}
The functions $\gBelow{\statex}$ and $\hBelow{\statex}$ represent admissible estimates of the cost-to-come to a state, $\statex \in \stateSet$, from the start and the cost-to-go from a state to the goal, respectively (i.e., they bound the true costs from below).
The function, $\fBelow{\statex}$, represents an admissible estimate of the cost of a path from $\xstart$ to $\goalSet$ constrained to pass through $\statex$, i.e., $\fBelow{\statex} := \gBelow{\statex} + \hBelow{\statex}$.
This estimate defines an informed set of states, $\fBelowSet := \setst{\statex\in\stateSet}{\fBelow{\statex} < \ccur}$, that could provide a solution better than the current best solution cost, $\ccur$ \citep{gammell_iros14,gammell_tro18}.
\squeezeWidowedWords

Let $\treeGraph := \pair{\vertexSet}{\edgeSet}$ be an \emph{explicit} tree with a set of vertices, $\vertexSet\subset\freeSet$, and edges, $\edgeSet = \set{\pair{\statev}{\statew}}$ for some $\statev,\, \statew \in \vertexSet$.
The function $\gAbove{\statex}$ then represents the cost-to-come to a state $\statex \in \stateSet$ from the start vertex given the current tree, $\treeGraph$.
A state not in the tree, or otherwise unreachable from the start, is assumed to have a cost-to-come of infinity.
It is important to recognize that these two functions will always bound the unknown true optimal cost to a state, $\gTrue{\cdot}$, i.e., $\forall \statex \in \stateSet,\,\gBelow{\statex} \leq \gTrue{\statex} \leq \gAbove{\statex}$.

The functions $\cBelow{\statex}{\statey}$ and $\cTrue{\statex}{\statey}$ represent an admissible estimate of the cost of an edge and the true cost of an edge between states $\statex,\, \statey \in \stateSet$, respectively.
Edges that intersect an obstacle are assumed to have a cost of infinity, and therefore $\forall \statex,\, \statey \in \stateSet,\,\, \cBelow{\statex}{\statey} \leq  \cTrue{\statex}{\statey} \leq \infty$.
\squeezeWidowedWords

The notation $\stateSet \setInsert \set{\statex} $ and $\stateSet \setRemove \set{\statex}$ is used to compactly represent the set compounding operations $\stateSet \gets \stateSet \cup \set{\statex}$ and $\stateSet \gets \stateSet \setminus \set{\statex}$, respectively.
As is customary, the minimum of an empty set is taken to be infinity.
\squeezeWidowedWords

\subsection[Initialization]{Initialization (\algolines{algo:bitstar}{startInit}{endInit})}\label{sec:bitstar:algo:init}
\ac{BITstar} begins searching a planning problem with the start, $\xstart$, in the spanning tree, $\treeGraph\coloneqq\pair{\vertexSet}{\edgeSet}$, and the goal states, $\goalSet$, in the set of unconnected states, $\unconnectedSet$ (\algolines{algo:bitstar}{initTree}{initSamples}).
This defines an implicit \ac{RGG} whose vertices consist of all states (i.e., $\vertexSet\cup\unconnectedSet$) and whose edges are defined by a distance function and an appropriate connection limit.
When the goal is a continuous region of the problem domain it will need to be discretized (e.g., sampled) before adding to the set of unconnected states.

The explicit spanning tree of this edge-implicit \ac{RGG} is built using two queues, a vertex expansion queue, $\vertexQueue$, and an edge evaluation queue, $\edgeQueue$.
These queues are sorted in order of potential solution quality through the current tree.
Vertices in the vertex queue, $\statev\in\vertexQueue$, are ordered by the sum of their current cost-to-come and an estimate of their cost-to-go, $\gAbove{\statev} + \hBelow{\statev}$.
Edges in the edge queue, $\pair{\statev}{\statex}\in\edgeQueue$, are sorted by the sum of the current-cost-to-come of their source vertex, an estimate of the edge cost, and an estimate of the cost-to-go of their target vertex, $\gAbove{\statev} + \cBelow{\statev}{\statex} + \hBelow{\statex}$.
Ties are broken in the vertex queue in favour of entries with the lowest cost-to-come through the current tree, $\gAbove{\statev}$, and in the edge queue in favour of the lowest cost-to-come through the current tree and estimated edge cost, $\gAbove{\statev} + \cBelow{\statev}{\statex}$, and then the cost-to-come through the current tree, $\gAbove{\statev}$.
These queues are initialized to contain all the vertices in the tree and an empty queue, respectively (\algoline{algo:bitstar}{initQueues}).
\squeezeWidowedWords%

To improve search efficiency, \ac{BITstar} tracks the vertices in the goal region, $\solnVertices$, the vertices that have never been expanded, $\unexpandedVertices$, the samples newly created during this batch, $\newSet$, and the current best solution, $\ccur$.
These are initialized to any vertices already in the solution (empty in all but the most trivial planning problems), the existing vertices, the existing samples, and the current best solution, respectively (\algolines{algo:bitstar}{initLabels}{initCmin}).

Initialized, \ac{BITstar} now searches the continuous planning problem by alternately building increasingly accurate implicit \ac{RGG} approximations (Section~\ref{sec:bitstar:algo:batch}) and searching these representations for explicit solutions in order of potential solution quality (Sections~\ref{sec:bitstar:algo:getEdge} and \ref{sec:bitstar:algo:addEdge}).

\begin{algorithm}[tp]
    \algorithmStyle
    \caption[\acs{BITstar}]%
    {\acs{BITstar}$\left(\xstart \in \freeSet,\; \goalSet\subset\freeSet\right)$}\label{algo:bitstar}
    \tikzmark{initTop}$\vertexSet \gets \set{\xstart};\;$
    $\edgeSet \gets \emptyset;\;$
    $\treeGraph=\pair{\vertexSet}{\edgeSet}$\label{algo:bitstar:initTree}\label{algo:bitstar:startInit}\;
    $\unconnectedSet \gets \goalSet$\label{algo:bitstar:initSamples}\;
    $\vertexQueue \gets \vertexSet;\;$
    $\edgeQueue \gets \emptyset$\label{algo:bitstar:initQueues}\;
    $\solnVertices \gets \vertexSet\cap\goalSet;\;$
    $\unexpandedVertices \gets \vertexSet;\;$
    $\newSet \gets \unconnectedSet$\label{algo:bitstar:initLabels}\; 
    $\ccur \gets \min_{\vgoal\in\solnVertices} \set {\gAbove{\vgoal}}$\label{algo:bitstar:initCmin} \label{algo:bitstar:endInit}\tikzmark{initBottom}\;
    \Repeat{$\mathtt{STOP}$}
    {
        \If{\tikzmark{batchTop}$\edgeQueue \equiv \emptyset$ {\rm \bf and} $\vertexQueue \equiv \emptyset$\label{algo:bitstar:emptyQueues}\label{algo:bitstar:startBatch}}
        {
            $\reuseSet \gets \mathtt{Prune}\left(\treeGraph,\unconnectedSet,\ccur\right)$\label{algo:bitstar:prune}\;
            $\samplingSet \gets \mathtt{Sample}\left(\samplesPerBatch, \xstart, \goalSet, \ccur\right)$\label{algo:bitstar:sample}\;
            $\newSet \gets \reuseSet \cup \samplingSet$\label{algo:bitstar:resetNewSamples}\;
            $\unconnectedSet \setInsert \newSet$\label{algo:bitstar:addNewSamples}\;
            $\vertexQueue \gets \vertexSet$\label{algo:bitstar:resetVertexQueue}\label{algo:bitstar:endBatch}\tikzmark{batchBottom}\;
        }
        
        \While{\tikzmark{queueTop}$\mathtt{BestQueueValue}\left(\vertexQueue\right)\leq\mathtt{BestQueueValue}\left(\edgeQueue\right)$\label{algo:bitstar:compareQueues}\label{algo:bitstar:startExpandVertex}\label{algo:bitstar:startNextEdge}}
        {
            $\mathtt{ExpandNextVertex}\left(\vertexQueue, \edgeQueue, \ccur\right)$\label{algo:bitstar:updateEdgeQueue}\label{algo:bitstar:endExpandVertex}\;
        }
        
        $\bestEdge \gets \mathtt{PopBestInQueue}\left(\edgeQueue\right)$\label{algo:bitstar:popEdgeQueue}\label{algo:bitstar:endNextEdge}\tikzmark{queueBottom}\;
        \If{\tikzmark{addTop}$\gAbove{\vmin}+\cBelow{\vmin}{\xmin}+\hBelow{\xmin}<\ccur$\label{algo:bitstar:canBeBetterSoln}\label{algo:bitstar:startProcess}}
        {
            \If{$\gAbove{\vmin}+\cBelow{\vmin}{\xmin}<\gAbove{\xmin}$\label{algo:bitstar:canBeBetterGraph}}
            {
                $\cedge \gets \cTrue{\vmin}{\xmin}$\label{algo:bitstar:calculateEdgeCost}\;
                \If{$\gAbove{\vmin}+\cedge+\hBelow{\xmin}<\ccur$\label{algo:bitstar:canRealEdgeBeBetterSoln}}
                {
                    \If{$\gAbove{\vmin}+\cedge<\gAbove{\xmin}$\label{algo:bitstar:improvesGraph}}
                    {
                        \If{$\xmin \in \vertexSet$\label{algo:bitstar:isInTree}}
                        {
                            $\vparent \gets \mathtt{Parent}\left(\xmin \right)$\label{algo:bitstar:findParent}\;
                            $\edgeSet \setRemove \set{\pair{\vparent}{\xmin}}$\label{algo:bitstar:rmParent}\;
                        }
                        \Else
                        {
                            $\unconnectedSet \setRemove \set{\xmin}$\label{algo:bitstar:rmSample}\;
                            $\vertexSet \setInsert \set{\xmin}$\;
                            $\vertexQueue \setInsert \set{\xmin}$\;
                            $\unexpandedVertices \setInsert \set{\xmin}$\label{algo:bitstar:addVertex}\;
                            \If{ $\xmin \in \goalSet$ \label{algo:bitstar:isSoln}\label{algo:bitstar:goalStart}}
                            {
                                $\solnVertices \setInsert \left\lbrace \xmin \right\rbrace$\label{algo:bitstar:addSoln}\label{algo:bitstar:goalEnd}\;
                            }
                        }
                        $\edgeSet \setInsert \set{\bestEdge}$\label{algo:bitstar:addEdge}\;
                        $\ccur \gets \min_{\vgoal\in\solnVertices} \set {\gAbove{\vgoal}}\tikzmark{right}$\label{algo:bitstar:cmin}\;
                    }
                }
            }
        }
        \Else
        {
            $\edgeQueue \gets \emptyset;\;$
            $\vertexQueue \gets \emptyset$\label{algo:bitstar:clearQueue}\label{algo:bitstar:endProcess}\tikzmark{addBottom}\;
        }
    }
    \Return{$\treeGraph$};
    \AddNote{initTop}{initBottom}{right}{\S\ref{sec:bitstar:algo:init}}%
    \AddNote{batchTop}{batchBottom}{right}{\S\ref{sec:bitstar:algo:batch}}%
    \AddNote{queueTop}{queueBottom}{right}{\S\ref{sec:bitstar:algo:getEdge}}%
    \AddNote{addTop}{addBottom}{right}{\S\ref{sec:bitstar:algo:addEdge}}%
\end{algorithm}%

\subsection[Batch Addition]{Batch Addition (\algolines{algo:bitstar}{startBatch}{endBatch})}\label{sec:bitstar:algo:batch}
\ac{BITstar} alternates between building an increasingly dense approximation of the continuous planning problem and searching this representation for a solution.
The approximation is updated whenever it has been completely searched (i.e., both queues are empty; \algoline{algo:bitstar}{emptyQueues}) by removing unnecessary states and adding a batch of new samples.
This avoids the computational cost of representing regions of the problem domain that cannot provide a better solution while increasing the accuracy of approximating the regions that can (i.e., the informed set).
This improving approximation allows \ac{BITstar} to almost-surely converge asymptotically to the optimal solution.

The approximation is pruned to the informed set by removing any states or edges that cannot improve the \emph{current} solution (\algoline{algo:bitstar}{prune}; Section~\ref{sec:bitstar:algo:prune}).
This reduces unnecessary complexity but may disconnect vertices in the informed set that cannot improve the solution solely because of their current connections.
These vertices are recycled as additional `new' samples in the batch so that they may be reconnected later if better connections are found.

The approximation is improved by adding $\samplesPerBatch$ new randomly generated samples from the informed set (\algoline{algo:bitstar}{sample}).
This can be accomplished with direct informed sampling \citep{gammell_iros14,gammell_tro18} or advanced rejection sampling \citeExample{kunz_icra16} as appropriate.

Each batch of states are labelled as $\newSet$ for the duration of that batch's search (\algoline{algo:bitstar}{resetNewSamples}).
This set is used to improve search efficiency and consists of both the newly generated samples and the recycled disconnected vertices.
\ac{BITstar} adds these new states to the set of unconnected states and initializes the vertex queue with all the vertices in the tree (\algolines{algo:bitstar}{addNewSamples}{resetVertexQueue}) to restart the search (Sections~\ref{sec:bitstar:algo:getEdge} and \ref{sec:bitstar:algo:addEdge}).

\subsection[Edge Selection]{Edge Selection (\algolines{algo:bitstar}{startNextEdge}{endNextEdge})}\label{sec:bitstar:algo:getEdge}
Graph-based search techniques often assume that finding and evaluating vertex connections is computationally inexpensive (e.g., given explicitly).
This is not true in sampling-based planning as finding vertex connections (e.g., the edges in the edge-implicit \acs{RGG}) requires performing a nearest-neighbour search and evaluating them requires checking for collisions and solving \acp{2BVP}, such as differential constraints.
\ac{BITstar} avoids these computational costs until required by using a lazy search procedure that delays both finding and evaluating connections in the \ac{RGG}.
Similar lazy techniques can be found in both advanced graph-based search and sampling-based planners \citep{bohlin_icra00,sanchez_ijrr02,helmert_jair06,branicky_icra01,cohen_rss14,hauser_icra15,salzman_tro16}.

Connections are found by using a vertex queue, $\vertexQueue$, ordered by potential solution quality.
This vertex queue delays processing a vertex (i.e., performing a nearest-neighbour search) until its outgoing connections \emph{could} be part of the best solution to the current graph.
Connections from all vertices are evaluated by using an edge queue, $\edgeQueue$, also ordered by potential solution quality.
This edge queue delays evaluating an edge (i.e., performing collision checks and solving \acp{2BVP}) until it \emph{could} be part of the best solution to the current graph.

A vertex in the vertex queue could be part of the best solution when it could provide an outgoing edge better than the best edge in the edge queue.
When the heuristic is consistent (e.g., the $L^2$ norm) the queue value of a vertex, $\statev\in\vertexQueue$, is a lower-bounding estimate of the queue value of its outgoing edges,
\begin{align*}
    \forall \statex\in\stateSet,\; \gAbove{\statev} + \hBelow{\statev} \leq \gAbove{\statev} + \cBelow{\statev}{\statex} + \hBelow{\statex}.
\end{align*}

The best edge at any iteration can therefore be found by processing the vertex queue until it is worse than the edge queue (\algoline{algo:bitstar}{compareQueues}).
This process of removing a vertex from the vertex queue and placing its outgoing edges in the edge queue is referred to as \emph{expanding} a vertex (\algoline{algo:bitstar}{updateEdgeQueue}; Section~\ref{sec:bitstar:algo:expandVertex}).
Once all necessary vertices are expanded, the best edge in the queue, $\bestEdge$, is removed (\algoline{algo:bitstar}{popEdgeQueue}) and used for this iteration of the search (Section~\ref{sec:bitstar:algo:addEdge}).\

The functions $\mathtt{BestQueueValue}\left(\cdot\right)$ and $\mathtt{PopBestInQueue}\left(\cdot\right)$ return the value of the element at the front of a queue and pop the element off the front of a queue, respectively.

\subsection[Edge Processing]{Edge Processing (\algolines{algo:bitstar}{startProcess}{endProcess})}\label{sec:bitstar:algo:addEdge}
\ac{BITstar} also uses heuristics to avoid expensive calculations when evaluating the best edge, $\bestEdge$.
An edge is added to the spanning tree if and only if
\begin{enumerate}
    \item an \emph{estimate} of its cost \emph{could} provide a better \emph{solution}, given the current tree (\algoline{algo:bitstar}{canBeBetterSoln}),\label{item:edge:estBetterSoln}
    \begin{align*}
        \gAbove{\vmin} + \cBelow{\vmin}{\xmin} + \hBelow{\xmin} < \ccur,%
    \end{align*}
    \item an \emph{estimate} of its cost \emph{could} improve the \emph{current tree} (\algoline{algo:bitstar}{canBeBetterGraph}),\label{item:edge:estBetterGraph}
    \begin{align*}
        \gAbove{\vmin} + \cBelow{\vmin}{\xmin} < \gAbove{\xmin},%
    \end{align*}
    \item its \emph{real} cost \emph{could} provide a better \emph{solution}, given the current tree (\algoline{algo:bitstar}{canRealEdgeBeBetterSoln}),\label{item:edge:realBetterSoln}
    \begin{align*}
        \gAbove{\vmin} + \cTrue{\vmin}{\xmin} + \hBelow{\xmin} < \ccur,%
    \end{align*}
    \item and its \emph{real} cost \emph{will} improve the \emph{current tree} (\algoline{algo:bitstar}{improvesGraph}),\label{item:edge:realBetterGraph}
    \begin{align*}
        \gAbove{\vmin} + \cTrue{\vmin}{\xmin} < \gAbove{\xmin}.%
    \end{align*}
\end{enumerate}
For collision-free edges, Conditions \ref{item:edge:estBetterSoln} and \ref{item:edge:realBetterSoln} are always true in the absence of a solution, while Conditions \ref{item:edge:estBetterGraph} and \ref{item:edge:realBetterGraph} are always true when the target of the edge, $\xmin$, is not in the spanning tree.

Checking if the edge could ever provide a better solution or improve the current tree (Conditions \ref{item:edge:estBetterSoln} and \ref{item:edge:estBetterGraph}) allows \ac{BITstar} to reject edges without calculating their true cost (\algolines{algo:bitstar}{canBeBetterSoln}{canBeBetterGraph}).
Condition \ref{item:edge:estBetterSoln} also provides a stopping condition for searching the current \ac{RGG}.
When an edge fails this condition so does the entire queue and both queues can be cleared to start a new batch (\algoline{algo:bitstar}{clearQueue}).
If the edge fails Condition \ref{item:edge:estBetterGraph} it is discarded and the iteration finishes.
If the edge passes both these conditions its true cost is calculated by performing a collision check and solving any \acp{2BVP} (\algoline{algo:bitstar}{calculateEdgeCost}).
Note that edges in collision are considered to have infinite cost.

Checking if the real edge could provide a better solution given the current tree (Condition \ref{item:edge:realBetterSoln}), allows \ac{BITstar} to reduce tree complexity by rejecting edges that could never improve the current solution (\algoline{algo:bitstar}{canRealEdgeBeBetterSoln}).
Checking if the real edge improves the current tree (Condition \ref{item:edge:realBetterGraph}), assures the cost-to-come of the explicit tree decreases monotonically (\algoline{algo:bitstar}{improvesGraph}).
If the edge fails either of these conditions it is discarded and the iteration finishes.

An edge passing all of these conditions is added to the spanning tree.
If the target vertex is already connected (\algoline{algo:bitstar}{isInTree}), then the edge represents a \emph{rewiring} and the current edge must be removed (\algolines{algo:bitstar}{findParent}{rmParent}).
Otherwise, the edge represents an \emph{expansion} of the tree and the target vertex must be moved from the set of unconnected states to the set of vertices, inserted into the vertex queue for future expansion, and marked as a never-expanded vertex (\algolines{algo:bitstar}{rmSample}{addVertex}).
The new vertex is also added to the set of vertices in the goal region if appropriate (\algolines{algo:bitstar}{isSoln}{addSoln}).
\squeezeWidowedWords

The new edge is then finally added to the tree (\algoline{algo:bitstar}{addEdge}) and the current best solution is updated as necessary (\algoline{algo:bitstar}{cmin}).
The search then continues by selecting the next edge in the queue (Section~\ref{sec:bitstar:algo:getEdge}) or increasing the approximation accuracy if the current \ac{RGG} cannot provide a better solution (Section~\ref{sec:bitstar:algo:batch}).

\begin{algorithm}[tb]
    \algorithmStyle
    \caption[ExpandNextVertex]%
    {$\mathtt{ExpandNextVertex}\left(\vertexQueue\subseteq\vertexSet,
                                    \vphantom{\edgeQueue\subseteq\vertexSet\times\left(\vertexSet\cup\stateSet\right),\; \ccur\in\positiveReal}\right.$\\
                                    $\left.\hspace{85pt}\vphantom{\vertexQueue\subseteq\vertexSet,}
                                \edgeQueue\subseteq\vertexSet\times\left(\vertexSet\cup\stateSet\right),\; \ccur\in\positiveReal\right)$
    }\label{algo:expand}
    $\vmin \gets \mathtt{PopBestInQueue}\left(\vertexQueue\right)$\label{algo:expand:popVertexQueue}\;
    \If{$\vmin\in\unexpandedVertices$\label{algo:expand:newForSamples}\label{algo:expand:startNearSamples}}
    {
        $\nearSet \gets \mathtt{Near}\left(\unconnectedSet,\vmin,\rbitstar\right)$\label{algo:expand:nearAllSamples}\;
    }
    \Else
    {
        $\nearSet \gets \mathtt{Near}\left(\newSet\cap\unconnectedSet,\vmin,\rbitstar\right)$\label{algo:expand:nearNewSamples}\label{algo:expand:endNearSamples}\;
    }
    $\edgeQueue \setInsert \left\lbrace \pair{\vmin}{\statex} \in \vertexSet\times \nearSet \;\; \middle| \;\; \gBelow{\vmin}
        \vphantom{+ \cBelow{\vmin}{\statex} + \hBelow{\statex} < \ccur}\right.$\label{algo:expand:insertNearSamples}\\
        \skipln$\left.\hspace{100pt}\vphantom{\pair{\vmin}{\statex} \in \vertexSet\times \nearSet \;\; \gBelow{\vmin}}
    {}+ \cBelow{\vmin}{\statex} + \hBelow{\statex} < \ccur\right\rbrace$\;
    
    \If{$\vmin\in\unexpandedVertices$\label{algo:expand:newForVertices}\label{algo:expand:startNearVertices}}
    {
        $\nearVertices \gets \mathtt{Near}\left(\vertexSet,\vmin,\rbitstar\right)$\label{algo:expand:nearVertices}\;

        $\edgeQueue \setInsert \left\lbrace \pair{\vmin}{\statew} \in \vertexSet \times \nearVertices \hspace{1.5pt} \middle| \hspace{1.5pt} \pair{\vmin}{\statew} \not\in \edgeSet,
            \vphantom{\gBelow{\vmin} + \cBelow{\vmin}{\statew} + \hBelow{\statew} < \ccur, \gBelow{\vmin} + \cBelow{\vmin}{\statew} < \gAbove{\statew}}\right.$\label{algo:expand:insertRewire}\\
        \skipln$\hspace{50pt}\gBelow{\vmin} + \cBelow{\vmin}{\statew} + \hBelow{\statew} < \ccur,$\\
            \skipln$\hspace{50pt}\left.\vphantom{\pair{\vmin}{\statew} \in \vertexSet \times \nearVertices \;\; \pair{\vmin}{\statew} \not\in \edgeSet, \gBelow{\vmin} + \cBelow{\vmin}{\statew} + \hBelow{\statew} < \ccur,}
        \gBelow{\vmin} + \cBelow{\vmin}{\statew} < \gAbove{\statew}\right\rbrace$\;
        
        $\unexpandedVertices \setRemove \set{\vmin}$\label{algo:expand:markOld}\label{algo:expand:endNearVertices}\;
    }
\end{algorithm}%
\subsection[Vertex Expansion]{Vertex Expansion (\algoline{algo:bitstar}{updateEdgeQueue}; \algo{algo:expand})}\label{sec:bitstar:algo:expandVertex}
The function $\mathtt{ExpandNextVertex}\left(\vertexQueue,\edgeQueue,\ccur\right)$ removes the front of the vertex queue (\algoline{algo:expand}{popVertexQueue}) and adds its outgoing edges in the \ac{RGG} to the edge queue.
The \ac{RGG} is defined using the results of \citet{karaman_ijrr11} to limit graph complexity while maintaining almost-sure asymptotic convergence to the optimum.
Edges exist between a vertex and the $\kbitstar$-closest states or all states within a distance of $\rbitstar$, with
\begin{align}
    \rbitstar &> \rbitstarmin,\nonumber\\
    \rbitstarmin &\coloneqq \left(2 \left(1 + \frac{1}{\dimension}\right)
                                \left(\frac{\min\left\lbrace\lebesgue{\stateSet},\lebesgue{\fBelowSet}\right\rbrace}{\unitBall{\dimension}}\right)\nonumber
                                \vphantom{\left(\frac{\log\left(\card{\vertexSet}+\card{\unconnectedSet} - \samplesPerBatch\right)}{\card{\vertexSet}+\card{\unconnectedSet}- \samplesPerBatch}\right)}\right.\\
                                &\hspace{40pt}
                                \left.\vphantom{2 \left(1 + \frac{1}{\dimension}\right)
                                                \left(\frac{\min\left\lbrace\lebesgue{\stateSet},\lebesgue{\fBelowSet}\right\rbrace}{\unitBall{\dimension}}\right)}
                                \left(\frac{\log\left(\card{\vertexSet}+\card{\unconnectedSet} - \samplesPerBatch\right)}{\card{\vertexSet}+\card{\unconnectedSet}- \samplesPerBatch}\right)
                            \right)^{\frac{1}{\dimension}},\label{eqn:radius:bitstar}
\end{align}
and
\begin{align}
    \kbitstar &> \kbitstarmin,\nonumber\\
    \kbitstarmin &\coloneqq e\left(1+\frac{1}{\dimension}\right)\log\left(\card{\vertexSet}+\card{\unconnectedSet} - \samplesPerBatch\right),\label{eqn:knearest:bitstar}
\end{align}
where $\card{\cdot}$ is the cardinality of a set, $\samplesPerBatch$ is the number of samples added in the last batch,  $\lebesgue{\cdot}$ is the Lebesgue measure of a set (e.g., the \textit{volume}), and $\unitBall{\dimension}$ is the Lebesgue measure of an $n$-dimensional unit ball.
Recent work has presented different expressions \citep{janson_ijrr15} and expressions for non-Euclidean spaces \citep{kleinbort_wafr16}.

This connection limit is calculated from the cardinality of the graph \emph{minus} the $\samplesPerBatch$ new samples to simplify proving almost-sure asymptotic optimality (Section~\ref{sec:anal}).
This lower bound will be large for the initial sparse batches but it can be thresholded with a maximum edge length, as is done by \ac{RRTstar}, i.e., $\rbitstar' \coloneqq \min\set{\rbitstarmax, \rbitstar}$.
The function $\mathtt{Near}$ returns the states that meet the selected \ac{RGG} connection criteria for a given vertex.

Every vertex in the tree is either expanded or pruned in every batch.
Processing all the outgoing edges from vertices would result in \ac{BITstar} repeatedly considering the same previously rejected edges.
This can be avoided by using the sets of never-expanded vertices, $\unexpandedVertices$, and new samples, $\newSet$, to add only \emph{previously unconsidered} edges to the edge queue.
\squeezeWidowedWords

Whether edges to unconnected samples are new depends on whether the source vertex has previously been expanded (\algoline{algo:expand}{newForSamples}).
If it has not been expanded then none of its outgoing connections have been considered and all nearby unconnected samples are potential descendants (\algoline{algo:expand}{nearAllSamples}).
If it has been expanded then any connections to `old' unconnected samples have already been considered and rejected and only the `new' samples are considered as potential descendants (\algoline{algo:expand}{nearNewSamples}).
The subset of these potential edges that could improve the current solution are added to the queue in both situations (\algoline{algo:expand}{insertNearSamples}).

Whether edges to connected samples (i.e., \emph{rewirings}) are new also depends on whether the source vertex has been expanded (\algoline{algo:expand}{newForVertices}).
If it has not been expanded then all nearby connected vertices are considered as potential descendants (\algoline{algo:expand}{nearVertices}).
The subset of these potential edges that could improve the current solution \emph{and} the current tree are added to the queue (\algoline{algo:expand}{insertRewire}) and the vertex is then marked as expanded (\algoline{algo:expand}{markOld}).

If a vertex has previously been expanded then no rewirings are considered.
Improvements in the tree may now allow a previously considered edge to improve connected vertices but considering these connections would require repeatedly reconsidering infeasible edges.
As in \ac{RRTstar}, this lack of \emph{propagated rewiring} has no effect on almost-sure asymptotic optimality.

\begin{algorithm}[tbp]
    \algorithmStyle
    \caption[Prune (\acs{BITstar})]%
    {$\mathtt{Prune}\left(\treeGraph=\pair{\vertexSet}{\edgeSet},\; \unconnectedSet\subset\stateSet,
            \vphantom{\ccur\in\positiveReal}\right.$\\
            $\left.\hspace{88pt}\vphantom{\treeGraph=\pair{\vertexSet}{\edgeSet},\; \unconnectedSet\subset\stateSet,}
        \ccur\in\positiveReal\right)$
    }\label{algo:prune}
    $\reuseSet \gets \emptyset$\label{algo:prune:resetReuse}\;
    $\unconnectedSet \setRemove \setst{\statex\in\unconnectedSet}{\fBelow{\statex}\geq\ccur}$\label{algo:prune:pruneSamples}\;
    \ForAll{$\statev\in\vertexSet$ {\rm \bf in\, order\, of\, increasing} $\gAbove{\statev}$\label{algo:prune:iterateGraph}}
    {
        \If{$\fBelow{\statev} > \ccur$ {\rm \bf or} $\gAbove{\statev}+\hBelow{\statev} > \ccur$\label{algo:prune:pruneCond}}
        {
            $\vertexSet \setRemove \set{\statev};\;$
            $\solnVertices \setRemove \set{\statev};\;$
            $\unexpandedVertices \setRemove \set{\statev}$\label{algo:prune:rmVertex}\;
            $\vparent \gets \mathtt{Parent}\left(\statev\right)$\label{algo:prune:getParent}\;
            $\edgeSet \setRemove \set{\pair{\vparent}{\statev}}$\label{algo:prune:rmEdge}\;
            \If{$\fBelow{\statev} < \ccur$\label{algo:prune:reuseCond}}
            {
                $\reuseSet \setInsert \set{\statev}$\label{algo:prune:addReuse}\;
            }
        }
    }
    \Return{$\reuseSet$}\label{algo:prune:returnReuse}\;
\end{algorithm}%
\subsection[Graph Pruning]{Graph Pruning (\algoline{algo:bitstar}{prune}; \algo{algo:prune})}\label{sec:bitstar:algo:prune}
The function, $\mathtt{Prune}\left(\treeGraph,\unconnectedSet,\ccur\right)$, reduces the complexity of both the approximation of the continuous planning problem (i.e., the implicit \ac{RGG}) and its search (i.e., the explicit spanning tree) by limiting them to the informed set.
It removes any states that can \emph{never} provide a better solution and disconnects any vertices that cannot provide a better solution \emph{given the current tree}.
Disconnected vertices that \emph{could} improve the solution with a better connection are reused as new samples in the next batch to maintain uniform sample density in the informed set. 
This assures that every vertex is either expanded or pruned in each batch as assumed by $\mathtt{ExpandNextVertex}$ to avoid reconsidering edges (Section~\ref{sec:bitstar:algo:expandVertex}).

The set of recycled vertices is initialized as an empty set (\algoline{algo:prune}{resetReuse}) and all unconnected states that cannot provide a better solution (i.e., are not members of the informed set) are removed (\algoline{algo:prune}{pruneSamples}).
The connected vertices are then incrementally pruned in order of increasing cost-to-come (\algoline{algo:prune}{iterateGraph}) by identifying those that can never provide a better solution or improve the solution given the \emph{current} tree (\algoline{algo:prune}{pruneCond}).
Vertices that fail either condition are removed from the tree by disconnecting their incoming edge and removing them from the vertex set and any labelling sets (\algolines{algo:prune}{rmVertex}{rmEdge}).

Any disconnected vertex that could provide a better solution (i.e., is a member of the informed set) is reused as a sample in the next batch (\algolines{algo:prune}{reuseCond}{addReuse}).
This maintains uniform sample density in the informed set and assures that vertices will be reconnected if future improvements allow them to provide a better solution.
This set of disconnected vertices is returned to \ac{BITstar} at the end of the pruning procedure (\algoline{algo:prune}{returnReuse}).

\subsection{Practical Considerations}\label{sec:bitstar:algo:pract}
\algos{algo:bitstar}{algo:prune} describe a generic version of \ac{BITstar} and leave room for a number of practical improvements depending on the specific implementation.

Searches (e.g., \algoline{algo:expand}{nearAllSamples}) can be implemented efficiently with appropriate datastructures that do not require an exhaustive global search \citeParenthetical{e.g., $k$-d trees or randomly transformed grids}{kleinbort_icra15}.
Pruning (\algoline{algo:bitstar}{prune}; \algo{algo:prune}) is computationally expensive and should only occur when a new solution has been found or limited to \emph{significant} changes in solution cost.

In an object-oriented programming language, many of the sets (e.g., $\newSet$, $\unexpandedVertices$) can be implemented more efficiently as labels.
The cost-to-come to a state in the current tree, $\gAbove{\cdot}$, can also be implemented efficiently using back pointers.

While queues can be implemented efficiently by using containers that sort on insertion, the value of elements in the vertex and edge queues will change when vertices are rewired.
There appears to be little practical difference between efficiently resorting the affected elements in these queues and only lazily resorting the queue before finishing to assure no elements have been missed.

Depending on the datastructure used for the edge queue, it may be beneficial to remove unnecessary entries when a new edge is added to the spanning tree, i.e., by adding
\begin{align*}
    \edgeQueue \setRemove &\left\lbrace \pair{\statev}{\xmin}\in\edgeQueue \;\;\middle|\;\; \gBelow{\statev}
                                \vphantom{+ \cBelow{\statev}{\xmin} \geq \gAbove{\xmin}}\right.\\
                               &\hspace{80pt}
                                \left.\vphantom{\pair{\statev}{\xmin}\in\edgeQueue \gBelow{\statev}}
                          {} + \cBelow{\statev}{\xmin} \geq \gAbove{\xmin}\right\rbrace
\end{align*}
after \algoline{algo:bitstar}{addEdge}.

\section{Analysis}\label{sec:anal}
\ac{BITstar} performance is analyzed theoretically using the results of \citet{karaman_ijrr11}.
It is shown to be probabilistically complete (Theorem~\ref{thm:pc}) and almost-surely asymptotically optimal (Theorem~\ref{thm:asao}).
Its search ordering is also shown to be equivalent to a lazy version of the ordering used in \ac{LPAstar} and \ac{TLPAstar} (Lemma~\ref{lem:lpa}).

\begin{thm}[Probabilistic completeness of \acs{BITstar}]\label{thm:pc}
    The probability that \ac{BITstar} finds a feasible solution to a given path planning problem, if one exists, when given infinite samples is one,
    \begin{align*}
        \liminf_{\totalSamples \to \infty}\probSymb &\left(\samplePathBIT\in\pathSet,\, \samplePathBIT\left(0\right) = \xstart,
                                                        \vphantom{\samplePathBIT\left(1\right) \in \goalSet}\right.\\
                                                        &\hspace{1in}\left.\vphantom{\samplePathBIT\in\pathSet,\, \samplePathBIT\left(0\right) = \xstart,}
                                                    \samplePathBIT\left(1\right) \in \goalSet\right) = 1,
    \end{align*}
    where $\totalSamples$ is the number of samples, $\samplePathBIT$ is the path found by \ac{BITstar} from those samples, and $\pathSet$ is the set of all feasible, collision-free paths.
\end{thm}
\begin{proof}
    Proof of Theorem~\ref{thm:pc} follows from the proof of almost-sure asymptotic optimality (Theorem~\ref{thm:asao}).
\end{proof}

\begin{thm}[Almost-sure asymptotic optimality of \acs{BITstar}]\label{thm:asao}
    The probability that \ac{BITstar} converges asymptotically towards the optimal solution of a given path planning problem, if one exists, when given infinite samples is one,
    \begin{align*}
        \prob{\limsup_{\totalSamples \to \infty}\pathCost{\samplePathBIT} = \pathCost{\bestPath}} = 1,
    \end{align*}
    where $\totalSamples$ is the number of samples, $\samplePathBIT$ is the path found by \ac{BITstar} from $\totalSamples$ samples and $\bestPath$ is optimal solution to the planning problem.
\end{thm}%
\begin{proof}
    Theorem~\ref{thm:asao} is proven by showing that \ac{BITstar} considers at least the same edges as \ac{RRTstar} for a sequence of states, $\samplesSet = \seqset{\statex_1,\statex_2,\ldots,\statex_\totalSamples}$, and connection limit, $\rbitstar \geq \rrrtstar$.
    
    \ac{RRTstar} incrementally builds a tree from a sequence of samples.
    For each state in the~sequence, $\statex_\counterk\in\samplesSet$, it considers the neighbourhood of earlier states that are within the connection~limit,
    \begin{align*}
        \nearSetk \coloneqq &\left\lbrace\statex_\counterj\in\samplesSet \;\; \middle| \;\; \counterj < \counterk,
                                \vphantom{\norm{\statex_\counterk-\statex_\counterj}{2} \leq \rrrtstar}\right.\\
                                &\left.\vphantom{\statex_\counterj\in\samplesSet \counterj < \counterk,}
                            \hspace{1.2in}\norm{\statex_\counterk-\statex_\counterj}{2} \leq \rrrtstar\right\rbrace.
    \end{align*}
    It selects the connection from this neighbourhood that minimizes the cost-to-come of the state and then evaluates the ability of connections from this state to reduce the cost-to-come of the other states in the neighbourhood.

    Given the same sequence of states, \ac{BITstar} groups them into batches of samples, $\samplesSet = \seqset{\batchSet_1, \batchSet_2, \ldots, \batchSet_\batches}$, where each batch is a set of $\samplesPerBatch < \totalSamples$ samples, e.g., $\batchSet_1 \coloneqq \set{\statex_1,\statex_2,\ldots,\statex_\samplesPerBatch}$.
    It incrementally builds a tree by processing this batched sequence of samples.
    For each state in the sequence, $\statey\in\batchSet_\counterk$, it considers the neighbourhood of states from the same or earlier batches within the connection limit,%
    \squeezeWidowedWords%
    \begin{align*}
        \nearSetk \coloneqq \setst{\statex\in\batchSet_\counterj}{\counterj \leq \counterk,\;\norm{\statey-\statex}{2} \leq \rbitstar}.
    \end{align*}
    It adds the edge in this neighbourhood to the tree that minimizes the cost-to-come of the state and considers all the outgoing edges to connect its neighbours.
    This set of edges contains all those considered by \ac{RRTstar} for an equivalent connection limit, $\rbitstar \geq \rrrtstar$, given that $\samplesPerBatch \geq 1$.
    
    As \eqref{eqn:radius:bitstar} uses the same connection radius for a batch that \ac{RRTstar} would use for the first sample in the batch and the connection radius of both are monotonically decreasing, this shows that \ac{BITstar} considers at least the same edges as \ac{RRTstar}.
    From these edges, \ac{BITstar} selects those that improve the cost-to-come of the target state and could currently provide a better solution as in \citet{karaman_icra11}.
    It is therefore almost-surely asymptotically optimal as stated in Theorem~\ref{thm:asao}.
\end{proof}%

\ac{BITstar} searches the \ac{RGG} in order of potential solution quality using an edge queue.
This is shown to be equivalent to the vertex queue ordering used by \ac{LPAstar}/\ac{TLPAstar} with a \emph{lazy} approximation of incoming edge costs (Lemma~\ref{lem:lpa}).
The search itself is not equivalent to \ac{LPAstar} as \ac{BITstar} does not reconsider outgoing connections from rewired vertices (i.e., it does not propagate rewirings).
It is instead a version of \ac{TLPAstar}.

\begin{lem}[The equivalent queue ordering of \acs{BITstar} and \acs{LPAstar}/\acs{TLPAstar}]\label{lem:lpa}
    The edge ordering in \ac{BITstar} that uses first the sum of a vertex's estimated cost-to-go, estimated incoming edge cost, and current cost-to-come of its parent,
    \begin{align*}
        \gAbove{\stateu} + \cBelow{\stateu}{\statev} + \hBelow{\statev},
    \end{align*}
    then the estimated cost-to-come of the vertex,
    \begin{align*}
        \gAbove{\stateu} + \cBelow{\stateu}{\statev},
    \end{align*}
    and then the cost-to-come of its parent,
    \begin{align*}
        \gAbove{\stateu},
    \end{align*}
    is equivalent to the vertex ordering in \ac{LPAstar} \citep{koenig_ai04} and \ac{TLPAstar} \citep{aine_ai16}.
\end{lem}%
\begin{proof}
    \ac{LPAstar} and \ac{TLPAstar} use a queue of vertices ordered lexicographically first on the solution cost constrained to go through the vertex and then the cost-to-come of the vertex.
    Both these terms are calculated for a vertex, $\statev \in \vertexSet$, considering all its possible incoming edges (referred to as the \textit{rhs-value} in \ac{LPAstar}), i.e.,
    \begin{align*}
    \min\limits_{\pair{\stateu}{\statev}\in\edgeSet}\left\lbrace\gAbove{\stateu}+\cTrue{\stateu}{\statev}\right\rbrace + \hBelow{\statev}\\
    \shortintertext{and}
    \min\limits_{\pair{\stateu}{\statev}\in\edgeSet}\left\lbrace\gAbove{\stateu}+\cTrue{\stateu}{\statev}\right\rbrace.
    \end{align*}%
    
    This minimum requires calculating the true edge cost between a vertex and all of its possible parents.
    This calculation is expensive in sampling-based planning (e.g., collision checking, differential constraints, etc.) and reducing its calculation is desirable.
    This can be achieved by incrementally calculating the minimum in the order given by an admissible heuristic estimate of edge cost.
    Considering edges into the vertex in order of increasing \emph{estimated} cost calculates a running minimum that can be stopped when the estimated cost through the next edge to consider is higher than the current minimum.
    
    \ac{BITstar} combines the minimum calculations for individual vertices into a single edge queue.
    This simultaneously calculates the minimum cost-to-come for each vertex in the current graph while expanding vertices in order of increasing estimated solution cost.
\end{proof}%

\section{Modifications and Extensions}\label{sec:mods}
The basic version of \ac{BITstar} presented in \algos{algo:bitstar}{algo:prune} can be modified and extended to include features that may improve performance for some planning applications.
Section~\ref{sec:mods:delay} presents a method to delay rewiring the tree until an initial solution is found.
This prioritizes exploring the \ac{RGG} to find solutions and may be beneficial in time-constrained applications.
Section~\ref{sec:mods:jit} presents a method to delay sampling until necessary.
This avoids approximating regions of the planning problem that are never searched, improves performance in large planning problems, and avoids the need to define \textit{a priori} limits in unbounded problems.
\squeezeWidowedWords

Section~\ref{sec:mods:drop} presents a method for \ac{BITstar} to occasionally remove unconnected samples while maintaining almost-sure asymptotic optimality.
This avoids repeated connection attempts to infeasible states and may be beneficial in problems where many regions of the free space are unreachable.
\squeezeWidowedWords

Section~\ref{sec:mods:sorrtstar} extends the idea of reducing the number of connections attempted per sample during an ordered search to develop \ac{SORRTstar}.
This version of \ac{RRTstar} uses batches of samples to order its search of the problem domain by potential solution quality, as in \ac{BITstar}, but uses a \texttt{steer} function and only makes one connection attempt per sample, as in \ac{RRTstar}.

\subsection{Delayed Rewiring}\label{sec:mods:delay}
Many robotic systems have a finite amount of computational time available to solve planning problems.
In these situations, rewiring the existing tree before an initial solution is found reduces the likelihood of \ac{BITstar} solving the given problem.
A method to delay rewirings until a solution is found is presented in \algo{algo:delay} as simple modifications to $\mathtt{ExpandNextVertex}$, with changes highlighted in red (cf. \algo{algo:expand}).
The rewirings are still performed once a solution is found and this method does not affect almost-sure asymptotic optimality.

Rewirings are delayed by separately tracking whether vertices are unexpanded to nearby unconnected samples, $\unexpandedVertices$, and unexpanded to nearby connected vertices, $\delayedVertices$.
This allows \ac{BITstar} to prioritize finding a solution by only considering edges to new samples until a solution is found and then improving the graph by considering rewirings from the existing vertices.

A vertex is moved from the never-expanded set to the delayed set when edges to its potential unconnected descendants are added to the edge queue (\algolines{algo:delay}{markOld}{markDelayed}).
This allows future expansions of the vertex to avoid old unconnected samples while tracking that the vertex's outgoing rewirings have not yet been considered.
Vertices in the delayed set are expanded as potential rewirings of other connected vertices once a solution is found (\algoline{algo:delay}{newForVertices}) and the delayed label is removed (\algoline{algo:delay}{markExpanded}).

This extension requires initializing and resetting $\delayedVertices$ along with the other labelling sets (e.g., \algoline{algo:bitstar}{initLabels} and \algoline{algo:prune}{rmVertex}).
This extension is included in the publicly available \ac{OMPL} implementation of \ac{BITstar}.
\begin{algorithm}[tp]
    \algorithmStyle
    \caption[ExpandNextVertex (Delayed Rewiring)]%
    {$\mathtt{ExpandNextVertex}\left(\vertexQueue\subseteq\vertexSet,
                                    \vphantom{\edgeQueue\subseteq\vertexSet\times\left(\vertexSet\cup\stateSet\right),\; \ccur\in\positiveReal}\right.$\\
                                    $\left.\hspace{85pt}\vphantom{\vertexQueue\subseteq\vertexSet,}
                                \edgeQueue\subseteq\vertexSet\times\left(\vertexSet\cup\stateSet\right),\; \ccur\in\positiveReal\right)$
    }\label{algo:delay}
    $\vmin \gets \mathtt{PopBestInQueue}\left(\vertexQueue\right)$\;
    \If{$\vmin\in\unexpandedVertices$}
    {
        $\nearSet \gets \mathtt{Near}\left(\unconnectedSet,\vmin,\rbitstar\right)$\;
        \newAlgoLine{$\unexpandedVertices \setRemove \set{\vmin}$\label{algo:delay:markOld}\;}
        \newAlgoLine{$\delayedVertices \setInsert \set{\vmin}$\label{algo:delay:markDelayed}\;}
    }
    \Else
    {
        $\nearSet \gets \mathtt{Near}\left(\newSet\cap\unconnectedSet,\vmin,\rbitstar\right)$\label{algo:delay:nearNewSamples}\;
    }
    $\edgeQueue \setInsert \left\lbrace \pair{\vmin}{\statex} \in \vertexSet\times \nearSet \;\; \middle| \;\; \gBelow{\vmin}
        \vphantom{+ \cBelow{\vmin}{\statex} + \hBelow{\statex} < \ccur}\right.$\\
        \skipln$\left.\hspace{100pt}\vphantom{\pair{\vmin}{\statex} \in \vertexSet\times \nearSet \;\; \gBelow{\vmin}}
    {}+ \cBelow{\vmin}{\statex} + \hBelow{\statex} < \ccur\right\rbrace$\;
    
    \If{$\vmin\in\newAlgoLine{\delayedVertices}$ \newAlgoLine{{\rm \bf and} $\ccur < \infty$}\label{algo:delay:newForVertices}}
    {
        $\nearVertices \gets \mathtt{Near}\left(\vertexSet,\vmin,\rbitstar\right)$\;

        $\edgeQueue \setInsert \left\lbrace \pair{\vmin}{\statew} \in \vertexSet \times \nearVertices \hspace{1.5pt} \middle| \hspace{1.5pt} \pair{\vmin}{\statew} \not\in \edgeSet,
            \vphantom{\gBelow{\vmin} + \cBelow{\vmin}{\statew} + \hBelow{\statew} < \ccur, \gBelow{\vmin} + \cBelow{\vmin}{\statew} < \gAbove{\statew}}\right.$\\
        \skipln$\hspace{50pt}\gBelow{\vmin} + \cBelow{\vmin}{\statew} + \hBelow{\statew} < \ccur,$\\
            \skipln$\hspace{50pt}\left.\vphantom{\pair{\vmin}{\statew} \in \vertexSet \times \nearVertices \;\; \pair{\vmin}{\statew} \not\in \edgeSet, \gBelow{\vmin} + \cBelow{\vmin}{\statew} + \hBelow{\statew} < \ccur,}
        \gBelow{\vmin} + \cBelow{\vmin}{\statew} < \gAbove{\statew}\right\rbrace$\;

        $\newAlgoLine{\delayedVertices} \setRemove \set{\vmin}$\label{algo:delay:markExpanded}\;
    }
\end{algorithm}%

\subsection{Just-in-Time (\acs{JIT}) Sampling}\label{sec:mods:jit}
Many robotic systems operate in environments that are unbounded (e.g., the outdoors).
These problems have commonly required using \textit{a priori} search limits to make the problem domain tractable.
Selecting these limits can be difficult and may prevent finding a solution (e.g., defining a domain that does not contain a solution) or reduce performance (e.g., defining a domain too large to search sufficiently).
A method to avoid these problems in \ac{BITstar} by generating samples \emph{just in time} (\acs{JIT}) is presented in \algo{algo:jit} and accompanying modifications to the main algorithm.
This modification generates samples only when needed by \ac{BITstar}'s search while still maintaining uniform sample density and almost-sure asymptotic convergence to the optimum.
This avoids approximating regions of the problem that are not used to find a solution and allows \ac{BITstar} to operate directly on large or unbounded planning problems.

\ac{BITstar} searches a planning problem by constructing and searching an implicit \ac{RGG} defined by a number of uniformly distributed samples (vertices) and their relative distances (edges).
These samples are given explicitly in \algos{algo:bitstar}{algo:prune} but are not used until they could be a descendant of a vertex in the tree.
This occurs when the samples are within the local neighbourhood of the vertex (Fig.~\ref{fig:jit}).
The cost associated with the informed set containing the neighbourhood of a vertex, $\nearSet$, for planning problems seeking to minimize path length,
\begin{align*}%
    \creqd \coloneqq \max_{\statex\in\nearSet}\set{\fBelow{\statex}},
\end{align*}%
is bounded from above as,
\begin{align*}%
    \creqd \leq \fBelow{\statev} + 2\rbitstar.
\end{align*}%

\begin{figure}[tbp]
    \centering
    \includegraphics[width=\columnwidth]{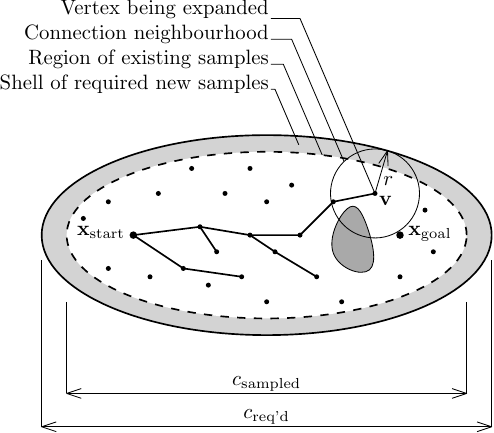}
    \caption[An illustration of \acf{JIT} sampling.]
    {%
        \figureCaptionStyle%
        An illustration of \acf{JIT} sampling.
        Samples are generated only when they are necessary for the expansion of a vertex, $\statev$, into the edge queue.
        This is accomplished while maintaining uniform sample density by an expanding informed set.
        The informed set is expanded by adding uniformly distributed samples in the prolate hyperspheroidal shell defined by the difference between the maximum heuristic value of the neighbourhood, $\creqd$, and the currently sampled informed set, $\csampled$.
    }
    \label{fig:jit}
\end{figure}%

\ac{JIT} sampling only generates samples when necessary to expand vertices by incrementally sampling this growing informed set.
It does this while maintaining uniform sample density by tracking the previously sampled size of the set, $\csampled$, and only generating the new samples necessary to increase it.
The function $\mathtt{UpdateSamples}\left(\statev, \csampled, \ccur\right)$ generates \ac{JIT} samples for vertex expansion in problems seeking to minimize path length (\algo{algo:jit}).
The required size of the sampled informed set, $\creqd$, is a function of the neighbourhood and the maximum size of the informed set (\algoline{algo:jit}{reqCost}).
If it is higher than the previously sampled cost, $\csampled$, then the local neighbourhood has not been completely sampled and new samples are generated (\algoline{algo:jit}{isNotSampled}).
If it is lower then the neighbourhood has already been sampled and no new samples are generated.

\begin{algorithm}[tbp]%
    \algorithmStyle
    \caption[UpdateSamples (Just-in-time Sampling)]%
    {$\mathtt{UpdateSamples}\left(\statev \in \vertexSet,\; \csampled\leq\ccur,
                                \vphantom{\ccur\in\positiveReal}\right.$\\
                                $\left.\hspace{130pt}\vphantom{\statev \in \vertexSet,\; \csampled\leq\ccur,}
                            \ccur\in\positiveReal\right)$
    }\label{algo:jit}
    $\creqd \gets \min{\left\lbrace\fBelow{\statev} + 2\rbitstar,\ccur\right\rbrace}$\;\label{algo:jit:reqCost}
    \If{$\creqd > \csampled$\label{algo:jit:isNotSampled}}
    {
        $\sampleMeasure \gets \phsMeasure\left(\creqd\right) - \phsMeasure\left(\csampled\right)$\label{algo:jit:reqVol}\;
        $\samplesPerBatch' \gets \sampleDensity \sampleMeasure$\label{algo:jit:reqNum}\;
        $\newSet \setInsert \mathtt{Sample}\left(\samplesPerBatch', \csampled, \creqd\right)$\label{algo:jit:sample}\;
        $\unconnectedSet \setInsert \newSet$\label{algo:jit:add}\;
        $\csampled \gets \creqd$\label{algo:jit:markSampled}\;
    }
\end{algorithm}%

The number of required new samples, $\samplesPerBatch'$, can be calculated from the chosen batch sample density, $\rho$, and the volume of the shell being sampled (\algolines{algo:jit}{reqVol}{reqNum}).
These samples are added to the set of new states and the set of unconnected states (\algolines{algo:jit}{sample}{add}).
Finally, the sampled cost is updated to reflect the new size of the sampled informed set (\algoline{algo:jit}{markSampled}).

The function $\phsMeasure\left(\cdot\right)$ calculates the measure of the prolate hyperspheroid defined by the start and goal with the given cost,
\begin{align*}
    \lebesgue{\fBelowSet} \leq \lebesgue{\phsSet} = \frac{\ccur \left( \ccur^2 - \cmin^2 \right)^{\frac{\dimension-1}{2}} \unitBall{\dimension}}{2^\dimension},
\end{align*}
where $\unitBall{\dimension}$ is the Lebesgue measure of an $\dimension$-dimensional unit ball,
\begin{align*}
    \unitBall{\dimension} \coloneqq \frac{\pi^{\frac{\dimension}{2}}}{\gammaFunc{\frac{\dimension}{2} + 1}},
\end{align*}
and $\gammaFunc{\cdot}$ is the gamma function, an extension of factorials to real numbers \citep{gamma_function}.
The function $\mathtt{Sample}\left(\samplesPerBatch',\csampled,\creqd\right)$ generates informed samples within the cost interval $\left[\csampled,\creqd\right)$ and may be implemented with rejection sampling.

Using \algo{algo:jit} requires modifying \algoAnd{algo:bitstar}{algo:expand}.
The $\mathtt{Sample}$ function of a batch (\algoline{algo:bitstar}{sample}) is replaced with an initialization of the sampled cost variable, $\csampled\gets0$, and $\mathtt{UpdateSamples}\left(\statev, \csampled, \ccur\right)$ is added to the front of $\mathtt{ExpandNextVertex}$ (\algo{algo:expand}).
This extension is included in the publicly available \ac{OMPL} implementation of \ac{BITstar}.
\squeezeWidowedLine

\subsection{Sample Removal}\label{sec:mods:drop}
\ac{BITstar} approximates continuous planning problems with an implicit \ac{RGG}.
It efficiently increases the accuracy of this approximation by focusing it to the informed set.
It alternately increases density by adding new samples and shrinks the set by searching the existing samples for better solutions.
This builds an explicit spanning tree that contains all states that could \emph{currently} provide a better solution but may not use every state in the \ac{RGG}.

States in an informed set may not be able to improve a solution for many reasons.
The approximation may be insufficiently accurate (i.e., low sample density) to capture difficult features (e.g., narrow passages) or represent sufficiently optimal paths.
The informed set may also include regions of the problem domain that cannot improve the solution due to unconsidered problem features (e.g., barriers separating free space) or because it is otherwise poorly chosen (i.e., low precision).
Unconnected samples in the first situation may later be beneficial to the search but samples in the second represent unnecessary computational cost.
Periodically removing these samples would reduce the complexity of the implicit \ac{RGG} and avoid repeatedly attempting to connect them to new vertices in the tree.

Unconnected samples can be removed while maintaining the requirements for almost-sure asymptotic optimality by modifying the \ac{RGG} connection limits to consider only uniformly distributed samples.
This can be accomplished by using the number of uniformly distributed samples added since the last removal of unconnected states in \eqref{eqn:radius:bitstar} and \eqref{eqn:knearest:bitstar}.
This simple extension is included in the publicly available \ac{OMPL} implementation of \ac{BITstar}.

\subsection[\acf{SORRTstar}]{\acf{SORRTstar}\footnotemark{}}\label{sec:mods:sorrtstar}%
\footnotetext{Pronounced \textit{sort star}.}%
\begin{algorithm}[tbp]%
    \algorithmStyle
    \caption[\acf{SORRTstar}]%
    {\acs{SORRTstar}$\left(\xstart\in\freeSet,\; \goalSet\subset\stateSet \right)$}\label{algo:sorrtstar}
    $\vertexSet \gets \left\lbrace \xstart \right\rbrace;\;$
    $\edgeSet \gets \emptyset;\;$
    $\treeGraph = \left( \vertexSet, \edgeSet \right)$\;
    $\solnVertices \gets \emptyset;\;$
    \newAlgoLine{$\sampleQueue \gets \emptyset$\;\label{algo:sorrtstar:init}}
    \For{$\iter = 1 \ldots \totalSamples$}
    {
        $\ccur \gets \min_{\vgoal\in\solnVertices} \set {\gAbove{\vgoal}}$\;
        \newAlgoLine
        {
            \If{$\sampleQueue\equiv\emptyset$\label{algo:sorrtstar:checkEmpty}}
            {
                $\sampleQueue \gets \mathtt{Sample}\left(\samplesPerBatch, \xstart, \goalSet, \ccur\right)$\label{algo:sorrtstar:addSamples}\;
            }
            $\xrand\gets\mathtt{PopBestInQueue}\left(\sampleQueue\right)$\label{algo:sorrtstar:popSample}\;
        }
        $\vnearest \gets \mathtt{Nearest}\left(\vertexSet, \xrand \right)$\;
        $\xnew \gets \mathtt{Steer}\left(\vnearest, \xrand\right)$\;
        \If{$\mathtt{CollisionFree}\left(\vnearest, \xnew\right)$}
        {
            \If{ $\xnew \in \goalSet$}
            {
                $\solnVertices \setInsert \left\lbrace \xnew \right\rbrace$\;
            }
            $\vertexSet \setInsert \left\lbrace \xnew \right\rbrace$\;
            $\nearVertices \gets \mathtt{Near}\left(\vertexSet, \xnew, \rrewire \right)$\;
            $\vmin \gets \vnearest$\;
            \ForAll{$\vnear \in \nearVertices$}
            {
                $\cnew \gets \gAbove{\vnear} + \cBelow{\vnear}{\xnew}$\;
                \If{$\cnew < \gAbove{\vmin} + \cTrue{\vmin}{\xnew}$}
                {
                    \If{$\mathtt{CollisionFree}\left( \vnear, \xnew \right)$}
                    {
                        $\vmin \gets \vnear$\;
                    }
                }
            }
            $\edgeSet \setInsert \left\lbrace\left(\vmin, \xnew \right)\right\rbrace$\;
            
            \ForAll{$\vnear \in \nearVertices$}
            {
                $\cnear \gets \gAbove{\xnew} + \cBelow{\xnew}{\vnear}$\;
                \If{$\cnear < \gAbove{\vnear}$}
                {
                    \If{$\mathtt{CollisionFree}\left( \xnew, \vnear \right)$}
                    {
                        $\vparent \gets \mathtt{Parent}\left(\vnear \right)$\;
                        $\edgeSet \setRemove \left\lbrace \left(\vparent, \vnear \right) \right\rbrace$\;
                        $\edgeSet \setInsert \left\lbrace \left( \xnew, \vnear\right)\right\rbrace$\;
                    }
                }
            }
            $\mathtt{Prune}\left(\vertexSet, \edgeSet, \ccur\right)$\;
        }
    }
    \Return{$\treeGraph$}\;
\end{algorithm}%

Approaching sampling-based planning as the search of an implicit \ac{RGG} motivates \ac{BITstar} to consider multiple connections to each sample.
Section~\ref{sec:mods:drop} presents a method to limit this number of attempts by periodically removing samples.
The natural extension of this idea is to consider only a single connection attempt per sample, as in \ac{RRTstar}.
This motivates the development of \ac{SORRTstar}, a version of \ac{RRTstar} that orders its search by potential solution quality by sorting batches of samples.

\ac{SORRTstar} is presented in \algo{algo:sorrtstar} as simple modifications of Informed \acs{RRTstar}, with changes highlighted in red \citeCompare{gammell_iros14,gammell_tro18}.
Instead of expanding the tree towards a randomly generated sample at each iteration, \ac{SORRTstar} extends the tree towards the best unconsidered sample in its current batch.
It accomplishes this by using a queue of samples, $\sampleQueue$, ordered by potential solution cost, $\fBelow{\cdot}$. %
This queue is filled with $\samplesPerBatch$ samples (\algolines{algo:sorrtstar}{checkEmpty}{addSamples}) and the search proceeds by expanding the tree towards the best sample in the queue (\algoline{algo:sorrtstar}{popSample}).
This orders the search for the $\samplesPerBatch$ samples in a batch, at which point a new batch of samples is generated and the search continues.
\squeezeWidowedWords

The function $\mathtt{PopBestInQueue}\left(\cdot\right)$ pops the best element off the front of a queue given its ordering.
A goal bias may be implemented in \ac{SORRTstar} by adding a small probability of sampling the goal instead of removing the best sample from the queue.
This algorithm is publicly available in \ac{OMPL}.

\ac{SORRTstar} can be viewed as a simplified version of \ac{BITstar} that only considers the best-possible edges.
Attempting to connect each sample once avoids the computational cost of repeated connection attempts to infeasible samples but still maintains some dependence on the sampling order.
High-utility samples (e.g., samples near the optimum) may be underutilized depending on the state of the tree when they are found.
This can become problematic if these samples have a low sampling probability (e.g., samples in narrow passages).
Making multiple connections attempts per sample and retaining samples for multiple batches allows \ac{BITstar} to exploit these useful samples more than algorithms such as \ac{SORRTstar}.
As seen in Section~\ref{sec:exp}, this results in different performance, especially in high state dimensions.

\section{Experiments}\label{sec:exp}
The benefits of ordering the search of continuous planning problems are demonstrated on simulated problems in $\real^2$, $\real^4$, $\real^8$, and $\real^{16}$ (Section~\ref{sec:exp:sim}) and one-{} and two-armed problems for \ac{HERB} (Section~\ref{sec:exp:herb}) using \ac{OMPL}.\footnotemark{}
\ac{BITstar} is compared to the \ac{OMPL} versions of \ac{RRT}, \acs{RRT}-Connect, \ac{RRTstar}, \RRTsharp{} (i.e., \RRTx{} with $\epsilon = 0$), \ac{FMTstar}, Informed \acs{RRTstar}, and \ac{SORRTstar}.
\footnotetext{The experiments were run on a laptop with $16$~GB of RAM and an Intel i7-4810MQ processor running Ubuntu 14.04 ($64$-bit).}

\begin{figure}[tb]
    \centering
    \includegraphics[width=\columnwidth]{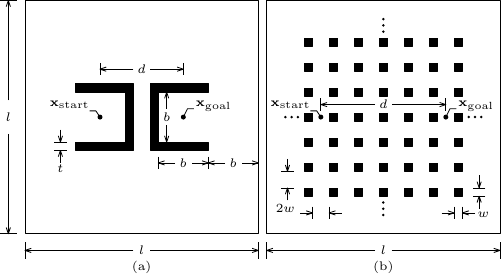}
    \caption%
    {%
        An illustration of the planning problems used in Sections~\ref{sec:exp:sim:bug} and \ref{sec:exp:sim:reg}.
        They are used to investigate algorithm performance on complex problems containing dual enclosures, (a), and many homotopy classes, (a), across state dimension.
        The problem dimensions in (a) were chosen to make the gaps symmetric, i.e., $b=0.6$, and $l=2.8$, for the chosen wall thickness, $t=0.1$.
        The width of the individual obstacles in (b) are chosen such that the start and goal states are $5$ `columns' apart in a problem domain of size $l=4$.
        For both problems, the distance between the start and goal, $d$, is $1$.
    }
    \label{fig:exp:defn}
\end{figure}%

All planners used the same tuning parameters and configurations where possible.
Planning time was limited to $1$~seconds, $10$~seconds, $30$~seconds, and $100$~seconds in $\real^2$, $\real^4$, $\real^8$, and $\real^{16}$ and $20$~seconds and $600$~seconds for \ac{HERB} ($\real^{7}$ and $\real^{14}$), respectively.
\acs{RRT}-style planners used a goal-sampling bias of $5\%$ and a maximum edge length of $\maxEdge=0.3$, $0.5$, $0.9$, and $1.7$ on the abstract problems ($\real^2$, $\real^4$, $\real^8$, and $\real^{16}$) and $0.7$ and $1.3$ on \ac{HERB}, respectively.
These values were selected experimentally to reduce the time required to find an initial solution on simple training problems.

The \ac{RRTstar} planners, \ac{FMTstar}, and \ac{BITstar} all used a connection radius equal to twice their lower bound (e.g., $\rrrtstar = 2\rrrtstarmin$) and approximated the Lebesgue measure of the free space with the measure of the entire planning problem.
The \ac{RRTstar} planners also used the ordered rewiring technique presented in \citet{perez_iros11}.
Informed \acs{RRTstar}, \ac{SORRTstar}, and \ac{BITstar} used the $L^2$ norm as estimates of cost-to-come and cost-to-go, direct informed sampling, and delayed pruning the graph until solution cost changed by more than $5\%$.
\ac{SORRTstar} and \ac{BITstar} both used $\samplesPerBatch=100$ samples per batch for all problems. %
\ac{BITstar} also thresholded its initial connection radius by using the same radius for both the first and second batches.

\subsection{Simulated Planning Problems}\label{sec:exp:sim}
The algorithms were tested on simulated problems in $\real^2$, $\real^4$, $\real^8$, and $\real^{16}$ on problems consisting of dual enclosures (Section~\ref{sec:exp:sim:bug}), many different homotopy classes (Section~\ref{sec:exp:sim:reg}), and randomly generated obstacles (Section~\ref{sec:exp:sim:rand}).
The planners were tested with $100$ different pseudo-random seeds on each problem and state dimension.
The solution cost of each planner was recorded every $10^{-4}$~seconds by a separate thread and the median was calculated from the $100$ trials by interpolating each trial at a period of $10^{-4}$~seconds, except for the problems in $\real^{16}$ where $10^{-3}$~seconds was used for both times.
The absence of a solution was considered an infinite cost for the purpose of calculating the median and infinite values were not plotted.
\begin{figure*}[p]
    \centering
    \includegraphics[width=\textwidth]{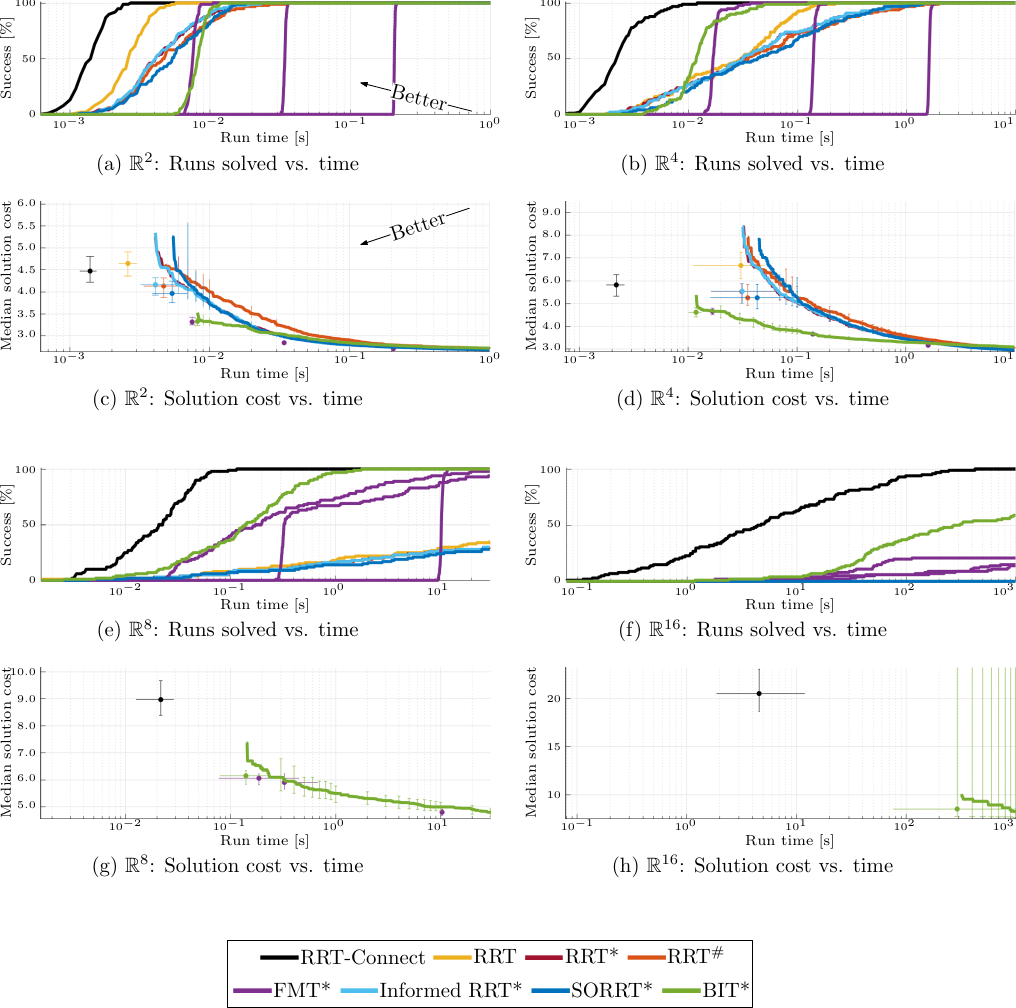}
    \caption[Planner performance versus time for the problem illustrated in Fig.~\ref{fig:exp:defn}a.]
    {%
        \figureCaptionStyle%
        Planner performance versus time for the problem illustrated in Fig.~\ref{fig:exp:defn}a.
        Each planner was run $100$ different times in $\real^2$, $\real^4$, $\real^8$, and $\real^{16}$ with run times limited to $1$, $10$, $30$, and $1000$ seconds, respectively. 
        The percentage of trials solved is plotted versus run time for each planner and presented in (a), (b), (e), and (f).
        The median path length is plotted versus run time for each planner and presented in (c), (d), (g), and (h), with unsuccessful trials assigned infinite cost.
        The error bars denote a nonparametric $99\%$ confidence interval on the median.
        The results show that \ac{BITstar} is competitive to other almost-surely asymptotically optimal planners in $\real^2$ and outperforms them in all other tested state dimensions.
        Note the difficulty of solving this problem in $\real^{8}$ and $\real^{16}$ in the available time.
    }
    \label{fig:exp:bugtrap}
\end{figure*}
\begin{figure*}[p]
    \centering
    \includegraphics[width=\textwidth]{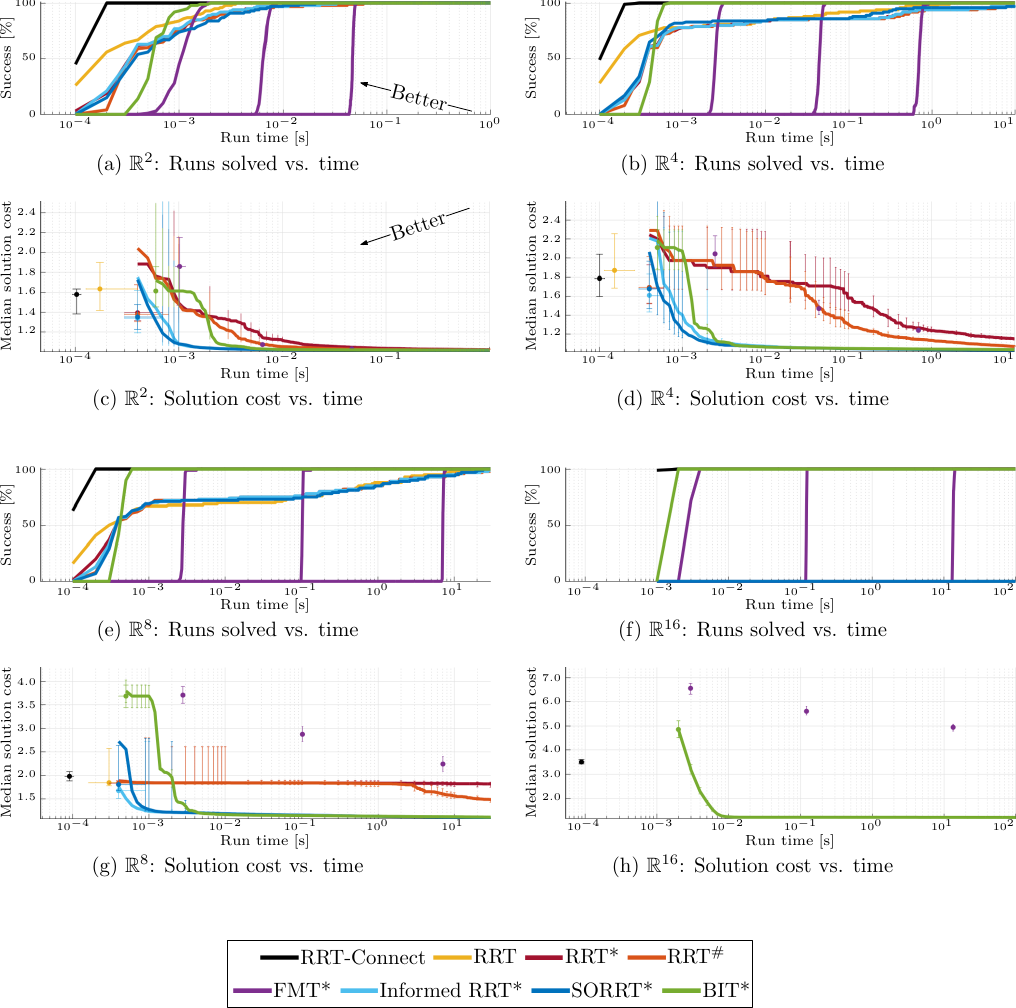}
    \caption[Planner performance versus time for the problem illustrated in Fig.~\ref{fig:exp:defn}b.]
    {%
        \figureCaptionStyle%
        Planner performance versus time for the problem illustrated in Fig.~\ref{fig:exp:defn}b.
        Each planner was run $100$ different times in $\real^2$, $\real^4$, $\real^8$, and $\real^{16}$ with run times limited to $1$, $10$, $30$, and $100$ seconds, respectively. 
        The percentage of trials solved is plotted versus run time for each planner and presented in (a), (b), (e), and (f).
        The median path length is plotted versus run time for each planner and presented in (c), (d), (g), and (h), with unsuccessful trials assigned infinite cost.
        The error bars denote a nonparametric $99\%$ confidence interval on the median.
        The results show that on this problem \ac{BITstar} finds a solution to all trials faster than all tested planners other than \acs{RRT}-Connect and outperforms other almost-surely asymptotically optimal planners in $\real^{16}$.
        Note that \acs{RRTstar}-based planners are not able to find solutions for this high state dimension in the time available.
    }
    \label{fig:exp:regular}
\end{figure*}
\begin{figure*}[p]
    \centering
    \includegraphics[width=\textwidth]{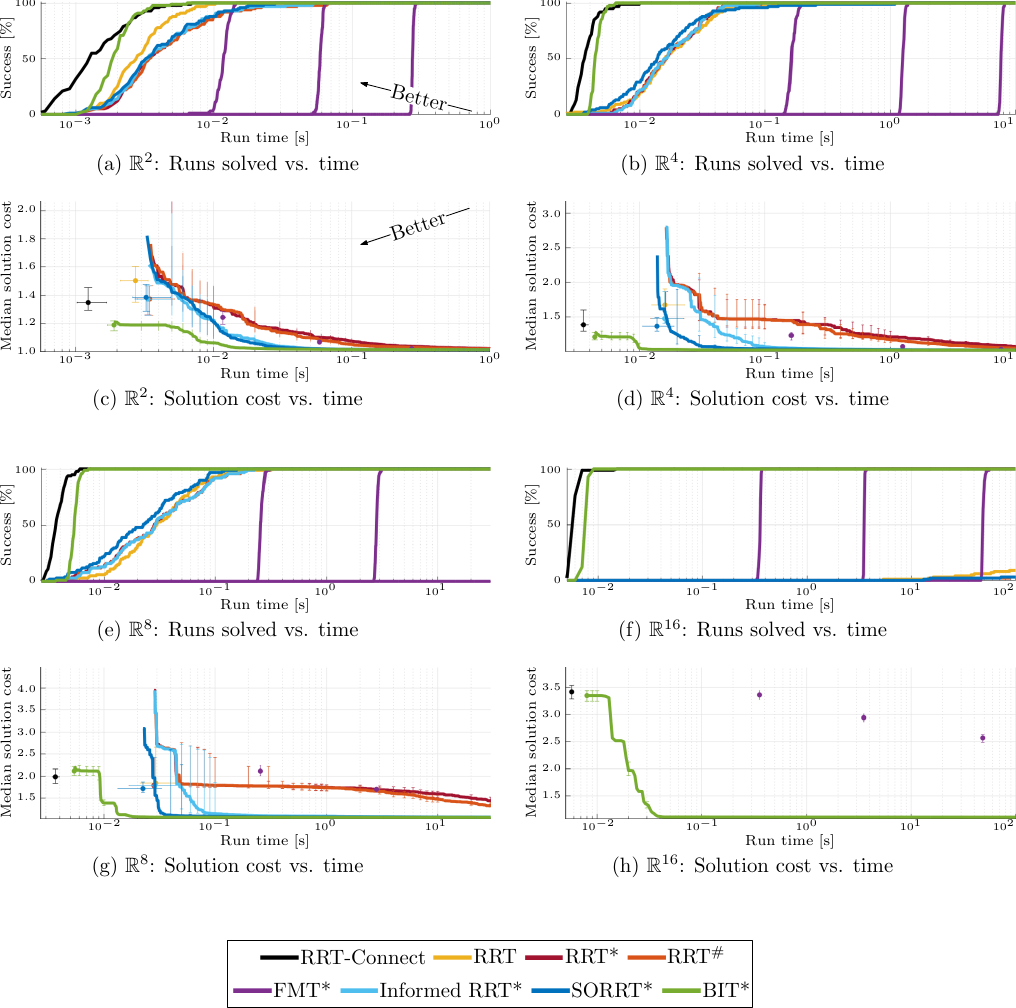}
    \caption[Planner performance versus time for a randomly generated problem.]
    {%
        \figureCaptionStyle%
        Planner performance versus time for a randomly generated problem.
        Each planner was run $100$ different times in $\real^2$, $\real^4$, $\real^8$, and $\real^{16}$ with run times limited to $1$, $10$, $30$, and $100$ seconds, respectively. 
        The percentage of trials solved is plotted versus run time for each planner and presented in (a), (b), (e), and (f).
        The median path length is plotted versus run time for each planner and presented in (c), (d), (g), and (h), with unsuccessful trials assigned infinite cost.
        The error bars denote a nonparametric $99\%$ confidence interval on the median.
        The results show that \ac{BITstar} outperforms other almost-surely asymptotically optimal planners in both ability to solve the problem and median solution cost.
        Note the increase in this difference with higher state dimension and the difficulty of solving the problem in $\real^{16}$ in the available time with \acs{RRTstar}-based planners.
    }
    \label{fig:exp:random}
\end{figure*}
\subsubsection{Dual-Enclosure Problems}\label{sec:exp:sim:bug}
The algorithms were tested on problems with two enclosures in $\real^2$, $\real^4$, $\real^8$, and $\real^{16}$ (Fig.~\ref{fig:exp:defn}a).
The problems consisted of a (hyper)cube of width $l=2.8$ with the start and goal located at $\left[-0.5,0,\ldots,0\right]^T$ and $\left[0.5,0,\ldots,0\right]^T$, respectively.
The enclosures were symmetric around the start and goal with a thickness of $t=0.1$ and openings of width $b=0.6$.
As all almost-surely asymptotically optimal planners struggled to solve the problem in $\real^{16}$, it was run for $1000$ seconds with recording and interpolation periods of $2\times10^{-3}$~seconds.

The results are presented in Fig.~\ref{fig:exp:bugtrap} with the percent of trials solved and the median solution cost plotted versus run time.
The results demonstrate the advantages of ordering the approximation and search of difficult planning problems.
\ac{BITstar} is competitive to other almost-surely asymptotically optimal planning algorithms in $\real^{2}$ and outperforms all algorithms other than \acs{RRT}-Connect in higher state dimensions.
In all dimensions, \ac{BITstar} finds a solution in every trial (i.e., attains $100\%$ success) sooner than the other anytime almost-surely asymptotically optimal planners.\squeezeWidowedWords

Specifically in $\real^2$, the median time required for \ac{BITstar} to find an initial solution is more than that of the \acs{RRTstar}-based planners (Fig.~\ref{fig:exp:bugtrap}c); however, once a solution is found, \ac{BITstar} finds better or equivalent solutions than the best performing \acs{RRTstar} planners at any given time.
\ac{FMTstar} slightly outperforms the other almost-surely asymptotically optimal planners but is not anytime.

The performance of \acs{RRTstar}-based planners decreases more rapidly with increasing state dimension on this problem than planners that process multiple samples such as \ac{BITstar} and \ac{FMTstar}.
\ac{BITstar} outperforms all planners other than \acs{RRT}-Connect in terms of success rate and median solution cost in $\real^{4}$ (Figs.~\ref{fig:exp:bugtrap}b and \ref{fig:exp:bugtrap}d).
This difference increases in $\real^{8}$ where the \ac{RRTstar}-based planners only find a solution in the available time in $35\%$ of the trials or less (Figs.~\ref{fig:exp:bugtrap}e and \ref{fig:exp:bugtrap}g).
All almost-surely asymptotically optimal planners struggle to solve the problem in $\real^{16}$ but \ac{BITstar} is the only one that finds a solution in more than $50\%$ of the trials (Figs.~\ref{fig:exp:bugtrap}f and \ref{fig:exp:bugtrap}h).
\squeezeWidowedWords

\subsubsection{Problems with Many Homotopy Classes}\label{sec:exp:sim:reg}
The algorithms were tested on problems with many homotopy classes in $\real^2$, $\real^4$, $\real^8$, and $\real^{16}$ (Fig.~\ref{fig:exp:defn}b).
The problems consisted of a (hyper)cube of width $l=4$ with the start and goal located at $\left[-0.5,0,\ldots,0\right]^T$ and $\left[0.5,0,\ldots,0\right]^T$, respectively.
The problem domain was filled with a regular pattern of axis-aligned (hyper)rectangular obstacles with a width such that the start and goal were $5$ `columns' apart.

\afterpage{\clearpage} %

The results are presented in Fig.~\ref{fig:exp:regular} with the percent of trials solved and the median solution cost plotted versus run time.
These results demonstrate both the advantages and disadvantages of attempting multiple connections per sample in problems with many disconnected obstacles.
In lower dimensions, informed algorithms that only make a single connection attempt per sample (e.g., Informed \ac{RRTstar} and \ac{SORRTstar}) have better median solution costs at any given time.
In higher dimensions, only the bidirectional \acs{RRT}-Connect and algorithms that make multiple connection attempts per sample (e.g., \ac{BITstar} and \ac{FMTstar}) find solutions in a reasonable amount of time.
In all dimensions, \ac{BITstar} finds a solution in every trial (i.e., attains $100\%$ success) sooner than every planner tested other than \acs{RRT}-Connect.

Specifically, \ac{BITstar} is the most likely almost-surely asymptotically optimal planning algorithm to solve the problem in $\real^{2}$, $\real^{4}$, and $\real^{8}$ (Figs.~\ref{fig:exp:regular}a, \ref{fig:exp:regular}b, and \ref{fig:exp:regular}e) but the informed, \acs{RRTstar}-based algorithms have better median solution costs (Figs.~\ref{fig:exp:regular}c, \ref{fig:exp:regular}d, \ref{fig:exp:regular}g).
This advantage disappears in $\real^{16}$ where \ac{BITstar} is the only anytime almost-surely asymptotically optimal planner to always solve the planning problem in the given time (Figs.~\ref{fig:exp:regular}f and \ref{fig:exp:regular}h).

This difference in performance may arise from the relative difficulty of planning tasks in different state dimensions.
It appears that in lower state dimensions the main challenge of this problem is avoiding obstacles.
In these situations, informed, \ac{RRTstar}-based algorithms will outperform \ac{BITstar} as they only make one connection attempt per sample.
In higher state dimensions, it appears that the exponential increase in problem measure (i.e., the curse of dimensionality) makes navigating towards the goal an equal challenge and \ac{BITstar} outperforms these other algorithms by considering multiple connection attempts per sample in an ordered fashion.

\subsubsection{Random Problems}\label{sec:exp:sim:rand}
The planners were tested on randomly generated problems in $\real^2$, $\real^4$, $\real^8$, and $\real^{16}$.
The worlds consisted of a (hyper)cube of width $l=2$ populated with approximately $75$ random axis-aligned (hyper)rectangular obstacles that obstruct at most one third of the environment.

For each state dimension, $10$ different random worlds were generated and the planners were tested on each with $100$ different pseudo-random seeds.
The true optima for these $10$ problems are different and unknown and there is no meaningful way to compare the results across problems.
Results from a representative problem are instead presented in Fig.~\ref{fig:exp:random} with the percent of trials solved and the median solution cost plotted versus computational time.

These experiments show that \ac{BITstar} generally finds better solutions faster than other sampling-based optimal planners and \ac{RRT} on these types of problems regardless of the state dimension.
It has a higher likelihood of having found a solution at a given computational time (Figs.~\ref{fig:exp:random}a, \ref{fig:exp:random}b, \ref{fig:exp:random}e, and \ref{fig:exp:random}f), and converges faster towards the optimum (Figs.~\ref{fig:exp:random}c, \ref{fig:exp:random}d, \ref{fig:exp:random}g, and \ref{fig:exp:random}h), with the relative improvement increasing with state dimension.
The only tested planner that found solutions faster than \ac{BITstar} was \acs{RRT}-Connect, a nonanytime planner that cannot converge to the optimum.

\begin{figure*}[p]
    \centering
    \includegraphics[width=\textwidth]{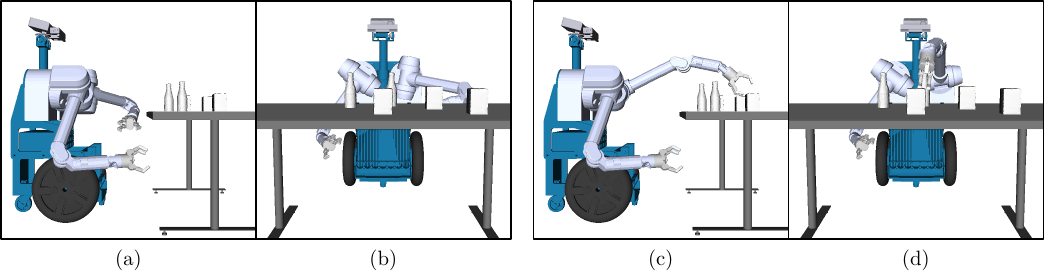}
    \caption[A one-armed motion planning problem for \acs{HERB} in $\real^7$.]
    {%
        \figureCaptionStyle%
        A one-armed motion planning problem for \acs{HERB} in $\real^7$.
        Starting at a position level with the table, (a) and (b), \acs{HERB}'s left arm must be moved in preparation for grasping a box on the far side of the table, (c) and (d).
    }
    \label{fig:exp:herb:1arm}
\end{figure*}
\begin{figure*}[p]
    \centering
    \includegraphics[width=\textwidth]{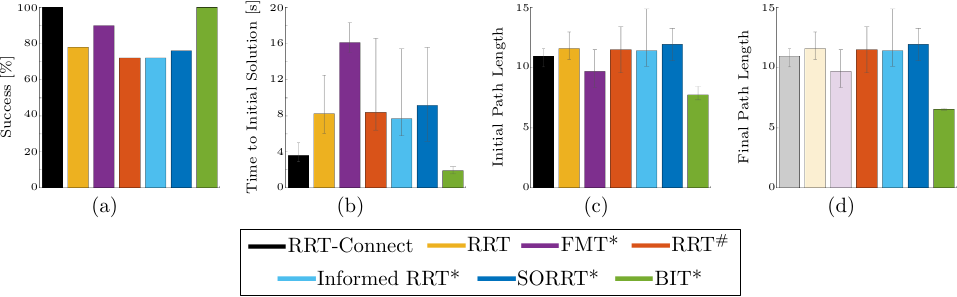}
    \caption[Results from $50$, $5$~second trials on the one-armed \acs{HERB} planning problem shown in Fig.~\ref{fig:exp:herb:1arm}.]
    {%
        \figureCaptionStyle%
        Results from $50$, $20$~second trials on the one-armed \acs{HERB} planning problem shown in Fig.~\ref{fig:exp:herb:1arm}.
        The percent of solutions solved, (a), the median time to an initial solution, (b), the median initial path length, (c), and the median final path length, (d), are presented with $99\%$ confidence intervals for each planner.
        Unsuccessful trials were assigned infinite time and cost.
        The inability of nonanytime planners (e.g., \acs{RRT}, \acs{RRT}-Connect, and \acs{FMTstar}) to use the remaining available time to improve their initial solution is denoted with diminished colour in (d), where present.
        \acs{BITstar} is the only almost-surely asymptotically optimal planner to solve all $50$ trials and does so in a time comparable to \acs{RRT}-Connect.
        It also finds significantly lower cost paths then all the other planners.
    }
    \label{fig:exp:herb:1arm:bars}
\end{figure*}
\begin{figure*}[p]
    \centering
    \includegraphics[scale=1]{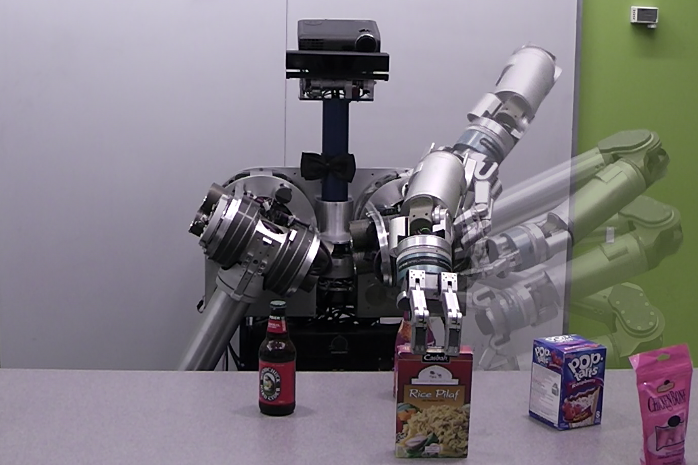}%
    \caption[A composite figure of \acs{HERB} executing a path found by \ac{BITstar} on a one-armed planning problem similar to Fig.~\ref{fig:exp:herb:1arm}.]
    {%
        \figureCaptionStyle%
        A composite figure of \acs{HERB} executing a path found by \ac{BITstar} on a one-armed planning problem similar to Fig.~\ref{fig:exp:herb:1arm}.
    }
    \label{fig:exp:herb:1arm:pic}
\end{figure*}%
\begin{figure*}[p]
    \centering
    \includegraphics[width=\textwidth]{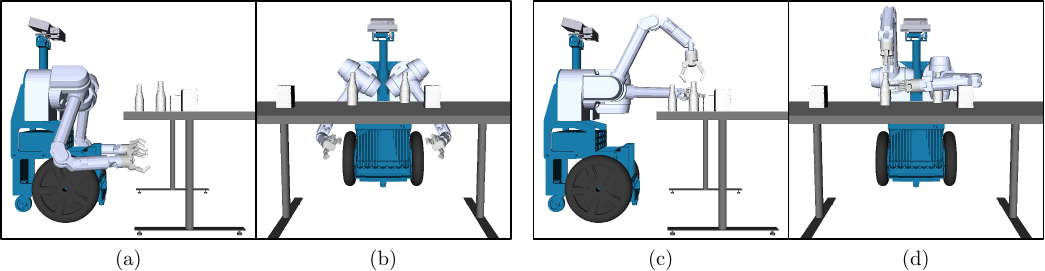}
    \caption[A two-armed motion planning problem for \acs{HERB} in $\real^{14}$.]
    {%
        \figureCaptionStyle%
        A two-armed motion planning problem for \acs{HERB} in $\real^{14}$.
        Starting under the table, (a) and (b), \acs{HERB}'s arms must be moved in preparation for opening a bottle, (c) and (d).
    }
    \label{fig:exp:herb:2arm}
\end{figure*}
\begin{figure*}[p]
    \centering
    \includegraphics[width=\textwidth]{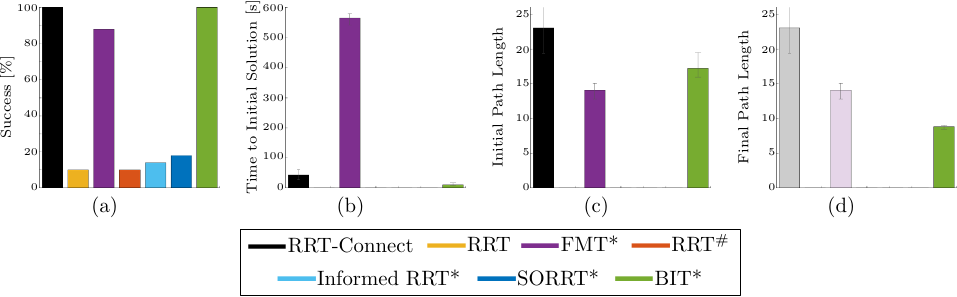}
    \caption[Results from $50$, $600$~second trials on the two-armed \acs{HERB} planning problem shown in Fig.~\ref{fig:exp:herb:2arm}.]
    {%
        \figureCaptionStyle%
        Results from $50$, $600$~second trials on the two-armed \acs{HERB} planning problem shown in Fig.~\ref{fig:exp:herb:2arm}.
        The percent of solutions solved, (a), the median time to an initial solution, (b), the median initial path length, (c), and the median final path length, (d), are presented with $99\%$ confidence intervals for each planner.
        Unsuccessful trials were assigned infinite time and cost.
        The inability of nonanytime planners (e.g., \acs{RRT}, \acs{RRT}-Connect, and \acs{FMTstar}) to use the remaining available time to improve their initial solution is denoted with diminished colour in (d), when present.
        \acs{BITstar} is the only anytime planner to solve all $50$ trials and does so in a time comparable to \acs{RRT}-Connect.
        It focuses the search to the informed set once a solution is found and the resulting increasingly dense \ac{RGG} allows it to find lower cost paths then all the other planners.
    }
    \label{fig:exp:herb:2arm:bars}
\end{figure*}
\begin{figure*}[p]
    \centering
    \includegraphics[scale=1]{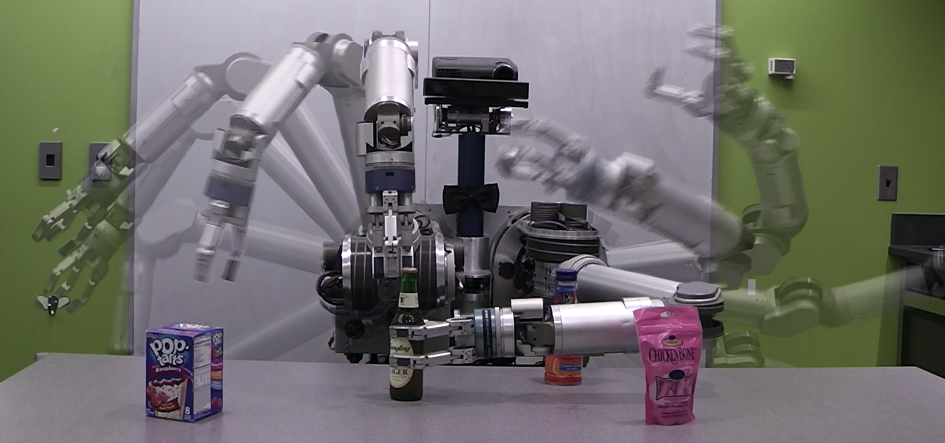}%
    \caption[A composite figure of \acs{HERB} executing a path found by \ac{BITstar} on a two-armed planning problem similar to Fig.~\ref{fig:exp:herb:2arm}.]
    {%
        \figureCaptionStyle%
        A composite figure of \acs{HERB} executing a path found by \ac{BITstar} on a two-armed planning problem similar to Fig.~\ref{fig:exp:herb:2arm}.
    }
    \label{fig:exp:herb:2arm:pic}
\end{figure*}%
\subsection{Path Planning for \acs{HERB}}\label{sec:exp:herb}
It is difficult to capture the challenges of actual high-dimensional planning in abstract worlds.
Two planning problems inspired by manipulation scenarios were created for \ac{HERB}, a 14-\ac{DOF} mobile manipulation platform.

Start and goal poses were chosen for one arm ($7$ \acp{DOF}, Section~\ref{sec:exp:herb:1arm}) and two arms ($14$ \acp{DOF}, Section~\ref{sec:exp:herb:2arm}) to define planning problems with the objective of minimizing path length through configuration space.
They were used to compare the \ac{OMPL} versions of \ac{RRT}, \acs{RRT}-Connect, \ac{FMTstar}, \RRTsharp{} (i.e., \RRTx{} with $\epsilon = 0$), Informed \acs{RRTstar}, \ac{SORRTstar},  and \ac{BITstar}.
The planners were run on each problem $50$ times while recording success rate, initial solution time and cost, and final cost.
Trials that did not find a solution were considered to have taken infinite time and have infinite path length, respectively, for the purpose of calculating medians.
The number of \ac{FMTstar} samples for both problems was chosen to use the majority of the available computational time.

\subsubsection{A One-Armed Planning Problem}\label{sec:exp:herb:1arm}
A planning problem was defined for \acs{HERB}'s left arm around a cluttered table (Fig.~\ref{fig:exp:herb:1arm}).
The arm starts folded at the elbow and held at approximately the level of the table (Figs.~\ref{fig:exp:herb:1arm}a and \ref{fig:exp:herb:1arm}b) and must be moved into position to grasp a box (Figs.~\ref{fig:exp:herb:1arm}c and \ref{fig:exp:herb:1arm}d).
The planners were given $20$~seconds of computational time to solve this $7$-\ac{DOF} problem with the objective of minimizing path length in configuration space.
\acs{FMTstar} used $\samplesPerBatch = 30$ samples.

The percentage of trials that successfully found a solution (Fig.~\ref{fig:exp:herb:1arm:bars}a), the median time and cost of the initial solution (Figs.~\ref{fig:exp:herb:1arm:bars}b and \ref{fig:exp:herb:1arm:bars}c) and the final cost (Fig.~\ref{fig:exp:herb:1arm:bars}d) were plotted for each planner.
Infinite values are not plotted.

The results show \ac{BITstar} finds solutions more often than other almost-sure asymptotically optimal planners on this problem (Fig.~\ref{fig:exp:herb:1arm:bars}a) and also finds better solutions faster than all tested algorithms, including \acs{RRT}-Connect (Figs.~\ref{fig:exp:herb:1arm:bars}b--\ref{fig:exp:herb:1arm:bars}d).
Fig.~\ref{fig:exp:herb:1arm:pic} presents a composite photograph of \ac{HERB} executing a path found by \ac{BITstar} for a similar problem.

\subsubsection{A Two-Armed Planning Problem}\label{sec:exp:herb:2arm}
A second planning problem was defined for both of \acs{HERB}'s arms moving around a cluttered table (Fig.~\ref{fig:exp:herb:2arm}).
The arms start at a neutral position with their forearms extended under the table (Figs.~\ref{fig:exp:herb:2arm}a and \ref{fig:exp:herb:2arm}b) and must be moved into position to open a bottle (Figs.~\ref{fig:exp:herb:2arm}c and \ref{fig:exp:herb:2arm}d).
The planners were given $600$~seconds of computational time to solve this $14$-\ac{DOF} problem with the objective of minimizing path length in configuration space.
\acs{FMTstar} used $\samplesPerBatch = 1750$ samples.

The percentage of trials that successfully found a solution (Fig.~\ref{fig:exp:herb:2arm:bars}a), the median time and cost of the initial solution (Figs.~\ref{fig:exp:herb:2arm:bars}b and \ref{fig:exp:herb:2arm:bars}c) and the final cost (Fig.~\ref{fig:exp:herb:2arm:bars}d) were plotted for each planner.
Infinite values are not plotted.

The results show that even when more computational time is available, \ac{BITstar} still finds solutions more often than other almost-surely asymptotically optimal planners (Fig.~\ref{fig:exp:herb:2arm:bars}a) and also finds initial solutions faster than all tested algorithms, including \acs{RRT}-Connect (Fig.~\ref{fig:exp:herb:2arm:bars}b).
As a nonanytime algorithm, \acs{FMTstar} is tuned to use the majority of the available time and finds a better initial solution (Fig.~\ref{fig:exp:herb:2arm:bars}c) but as \ac{BITstar} is able to improve its solution it still finds a better final path after approximately the same amount of time (Fig.~\ref{fig:exp:herb:2arm:bars}d).
Fig.~\ref{fig:exp:herb:2arm:pic} presents a composite photograph of \ac{HERB} executing a path found by \ac{BITstar} for a similar problem.
\squeezeWidowedWords

\section{Discussion \& Conclusion}\label{sec:fin}
Most planning algorithms discretize the search space of continuous path planning problems.
Popular approaches in robotics include \textit{a priori} graph-{} or anytime sampling-based approximations.
Both of these approaches are successful but have important limitations.

\textit{A priori} graphs approximate a problem before it is searched.
Doing so `correctly' is challenging since the relationship between resolution and search performance depends on the specific features of a planning problem (e.g., the size and arrangement of obstacles).
If the chosen approximation is insufficient (e.g., a sparse graph) it may preclude finding a (suitable) solution but if it is excessive (e.g., a dense graph) it may make finding a solution prohibitively expensive.

Graph-based searches are effective path planning techniques despite these limitations.
This is because informed algorithms, such as A*, use heuristics to order their search by potential solution quality.
This not only finds the optimal solution to the given representation (i.e., it is \textit{resolution optimal}) but does so by expanding the minimum number of vertices for the chosen heuristic \citeParenthetical{i.e., it is \textit{optimally efficient}}{hart_tssc68}.

Anytime sampling-based planners alternatively build approximations that increase in resolution.
This avoids the need to select a representation \textit{a priori} and allows them to be run indefinitely until a (suitable) solution is found.
\ac{RRTstar} has a unity probability of finding a solution, if one exists, with an infinite number of samples (i.e., it is \textit{probabilistically complete}) and finds continuously improving solutions \citeParenthetical{i.e., it is \textit{almost-surely asymptotically optimal}}{karaman_ijrr11}.

This has made them effective planning techniques despite the fact their search is often inefficient.
Incremental sampling-based planners tend to explore the entire problem domain equally (i.e., they are \textit{space filling}).
This wastes computational effort on regions of the problem domain that are not needed to find the solution and can be prohibitively expensive in large planning problems and/or high state dimensions.

Previous attempts to unify these two approaches have been incomplete.
They either sacrifice anytime resolution (e.g., \ac{RAstar} and \ac{FMTstar}), order the search on metrics other than solution cost and waste computational effort (e.g., \ac{SBAstar}), or do not order all aspects of the problem domain search (e.g., \RRTsharp{} and \RRTx{}).
This paper demonstrates that these tradeoffs are unnecessary and that anytime sampling-based planners can be directly and solely ordered by a heuristic estimate of solution cost.

\ac{BITstar} directly unifies informed graph-based search and sampling-based planning (Section~\ref{sec:bitstar}).
It uses heuristics and batches of random samples to simultaneously build anytime approximations of continuous planning problems and search these approximations in order of potential solution quality.
This avoids the computational costs of both searching an improper approximation (graph-based search) and performing an unordered search of the problem domain (sampling-based planning).
A version of \ac{BITstar} is publicly available in \ac{OMPL}.

\ac{BITstar} approximates the search space by using \emph{batches} of samples to define increasingly dense implicit \acp{RGG}.
Building this approximation with batches of \emph{multiple samples} allows each search to be ordered by potential solution quality, as in A*.
Building this approximation from \emph{multiple batches} of samples allows it to be improved indefinitely until it contains a suitable solution, as in \ac{RRTstar}.
As a result, \ac{BITstar} is probabilistically complete and almost-surely asymptotically optimal (Section~\ref{sec:anal}).

\begin{figure*}[tp]
    \centering
    \includegraphics[width=\textwidth]{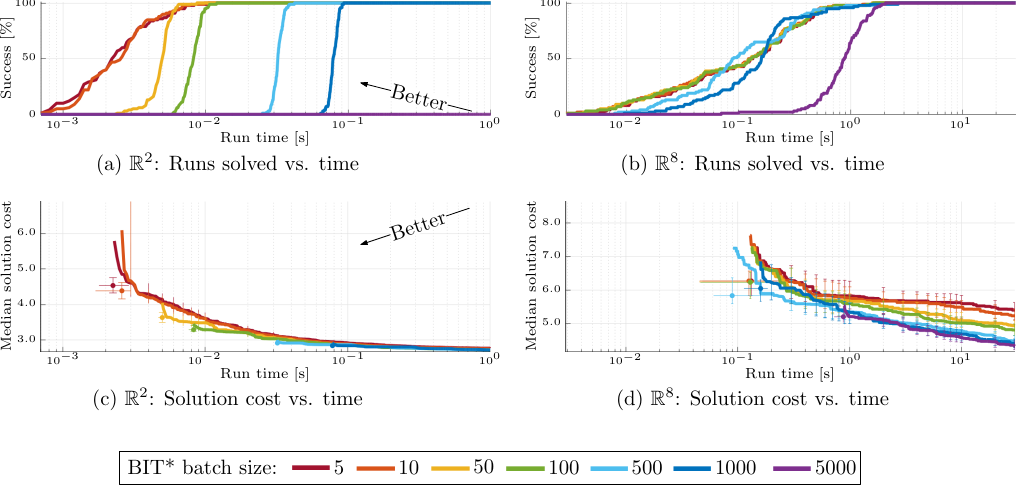}
    \caption[The performance versus time of \ac{BITstar} with various batch sizes on the problem illustrated in Fig.~\ref{fig:exp:defn}a.]
    {%
        \figureCaptionStyle%
        Planner performance versus time of \ac{BITstar} with various batch sizes on the problem illustrated in Fig.~\ref{fig:exp:defn}a.
        Each configuration was run $100$ different times in $\real^2$ and $\real^8$ with run times limited to $1$ and $30$ seconds, respectively. 
        The percentage of trials solved is plotted versus run time and presented in (a) and (b).
        The median path length is plotted versus run time and presented in (c) and (d) with unsuccessful trials assigned infinite cost.
        The error bars denote a nonparametric $99\%$ confidence interval on the median.
        It appears that decreasing the batch size decreases the median solution time towards a problem-specific threshold.
        The median initial solution time was the same in $\real^{2}$ for batch sizes of $5$ and $10$ samples but progressively higher for larger batches, (a).
        It was equivalent for batches of $5$, $10$, $50$, and $100$ samples in $\real{4}$ (not shown) and for all tested batch sizes other than $5000$ in $\real^{8}$, (b).
        Decreasing the batch size also appears to decrease the rate of convergence towards the optimum, with the effect becoming more pronounced in higher state dimensions, (d).
        It is not clear how universal these relationships are between obstacle configurations.
    }
    \label{fig:fin:batch}
\end{figure*}

This simultaneous approximation and search is done efficiently by using heuristics.
The approximation is focused to the regions of the planning problem that could provide better solutions, as in Informed \acs{RRTstar}.
The denser approximations are searched efficiently by reusing previous information, as in \ac{TLPAstar}.
Edge calculations (e.g., \acp{2BVP} and collisions checks) are delayed until necessary, as in \textit{lazy} versions of both graph-based searches and sampling-based planners \citeExample{bohlin_icra00,sanchez_ijrr02,helmert_jair06,branicky_icra01,cohen_rss14,hauser_icra15,salzman_tro16}.

A brief set of extensions to \ac{BITstar} are presented (Section~\ref{sec:mods}).
These include prioritizing an initial solution, avoiding the need to define \textit{a priori} search limits in unbounded problems, and avoiding unreachable areas of the problem domain.
These ideas also motivate the development of \ac{SORRTstar} as an extension of batch-ordered search to the algorithmic simplicity of \ac{RRTstar} (\algo{algo:sorrtstar}).
A version of \ac{SORRTstar} is publicly available in \ac{OMPL}.

The benefits of \ac{BITstar} are demonstrated experimentally on abstract planning problems and simulated experiments for \ac{HERB} (Section~\ref{sec:exp}).
The results highlight the advantages and disadvantages of both using an ordered search and considering multiple connections per sample.
As state dimension increases, \ac{BITstar} becomes more likely to have found a solution and generally finds better solutions faster than the other almost-surely asymptotically optimal planners.
\squeezeWidowedWords

The experiments also highlight the relative sensitivity of anytime planners to their tuning parameters.
The performance of \acs{RRT}-style planners depends heavily on the maximum edge length, $\maxEdge$, and achieving the best performance requires tuning it for the problem size, dimension, and even obstacle characteristics.
Alternatively, the same batch size was used for \ac{BITstar} on all the tested problems even though further tuning on specific problems could provide better performance (Fig.~\ref{fig:fin:batch}).
This result should motivate future research on more advanced sample addition procedures, including variable and adaptive batch sizes.
\squeezeWidowedWords

Using heuristics to avoid unnecessary edge evaluations allows \ac{BITstar} to spend more computational effort on the edges that are evaluated.
\citet{xie_icra15} show that a \ac{2BVP} solver can be used to calculate edges for \ac{BITstar} for problems with differential constraints.
They find that doing so is competitive to state-of-the-art optimal sampling-based techniques that are explicitly designed to avoid solving \acp{2BVP}.
\citet{choudhury_icra16} show that a path optimizer \citeParenthetical{i.e., \acs{CHOMP}}{zucker_ijrr13} can be used on potential edges in \ac{BITstar}.
This provides a method to exploit local problem information (i.e., cost gradients) to propose higher-quality edges and improve performance.

\ac{BITstar} is described as using a \acs{LPAstar} \emph{ordering} to efficiently search an incrementally built (i.e., changing) \ac{RGG} embedded in a continuous planning problem.
While the search order is the same, it is important to note a key difference in how these two algorithms reuse information.
When \ac{LPAstar} updates the cost-to-come of a vertex it reconsiders the cost-to-come of all possibly descendent vertices. %
This can be prohibitively expensive in large graphs and the results of \ac{RRTstar} demonstrate that this propagation is unnecessary for a planner to almost-surely converge asymptotically to the optimum.
By not propagating these changes, \ac{BITstar} performs a truncated rewiring similar to \ac{TLPAstar}.
\squeezeWidowedWords

This paper demonstrates the benefits of unifying informed graph-based search and sampling-based planning.
Using incremental search techniques to efficiently search an increasingly dense \ac{RGG} allows \ac{BITstar} to outperform existing anytime almost-surely asymptotically optimal planners.
These results will hopefully motivate further research into combining graph-based search and sampling-based planning.
Of particular interest would be probabilistic statements about search efficiency analogous to the formal statements for A*.

There is also a clear opportunity to consider different anytime approximations, such as deterministic sampling \citep{janson_ijrr18} or adaptive meshes \citep{yershov_ijrr16}, and more advanced graph-based-search techniques, such as \citeAcronymDefn{ARAstar}{likhachev_ai08} and \citeAcronymDefn{MHAstar}{aine_ijrr15}, to further accelerate the search performance of \ac{BITstar}.
\squeezeWidowedWords%

\begin{acks}
We would like to thank the editorial board for considering this manuscript and the reviewers for their detailed comments and dedicated efforts to improve it.
We would also like to thank Christopher Dellin, Michael Koval, and Jennifer King for help running the \acs{HERB} experiments.
\end{acks}

\begin{funding}
This research was supported by contributions from the \ac{NSERC} through the \ac{NCFRN}, the Ontario Ministry of Research and Innovation's Early Researcher Award Program, and the \ac{ONR} Young Investigator Program.
\end{funding}

\bibliographystyle{SageH}
\bibliography{TR-2014-JDG008}

\begin{thebibliography}{62}
\providecommand{\natexlab}[1]{#1}
\providecommand{\url}[1]{\texttt{#1}}
\providecommand{\urlprefix}{URL }
\expandafter\ifx\csname urlstyle\endcsname\relax
  \providecommand{\doi}[1]{DOI:\discretionary{}{}{}#1}\else
  \providecommand{\doi}{DOI:\discretionary{}{}{}\begingroup
  \urlstyle{rm}\Url}\fi

\bibitem[{Aine and Likhachev(2016)}]{aine_ai16}
Aine S and Likhachev M (2016) Truncated incremental search.
\newblock \emph{Artificial Intelligence} 234: 49--77.
\newblock \doi{10.1016/J.ARTINT.2016.01.009}.

\bibitem[{Aine et~al.(2015)Aine, Swaminathan, Narayanan, Hwang and
  Likhachev}]{aine_ijrr15}
Aine S, Swaminathan S, Narayanan V, Hwang V and Likhachev M (2015)
  Multi-heuristic {A*}.
\newblock \emph{The International Journal of Robotics Research (IJRR)}
  35(1--3): 224--243.
\newblock \doi{10.1177/0278364915594029}.
\newblock \urlprefix\url{http://ijr.sagepub.com/content/35/1-3/224.abstract}.

\bibitem[{Akgun and Stilman(2011)}]{akgun_iros11}
Akgun B and Stilman M (2011) Sampling heuristics for optimal motion planning in
  high dimensions.
\newblock In: \emph{Proceedings of the IEEE/RSJ International Conference on
  Intelligent Robots and Systems (IROS)}. pp. 2640--2645.
\newblock \doi{10.1109/IROS.2011.6095077}.

\bibitem[{Arslan and Tsiotras(2013)}]{arslan_icra13}
Arslan O and Tsiotras P (2013) Use of relaxation methods in sampling-based
  algorithms for optimal motion planning.
\newblock In: \emph{Proceedings of the IEEE International Conference on
  Robotics and Automation (ICRA)}. pp. 2421--2428.
\newblock \doi{10.1109/ICRA.2013.6630906}.

\bibitem[{Arslan and Tsiotras(2015)}]{arslan_icra15}
Arslan O and Tsiotras P (2015) Dynamic programming guided exploration for
  sampling-based motion planning algorithms.
\newblock In: \emph{Proceedings of the IEEE International Conference on
  Robotics and Automation (ICRA)}. pp. 4819--4826.
\newblock \doi{10.1109/ICRA.2015.7139869}.

\bibitem[{Arslan and Tsiotras(2016)}]{arslan_cdc16}
Arslan O and Tsiotras P (2016) Incremental sampling-based motion planners using
  policy iteration methods.
\newblock In: \emph{Proceedings of the IEEE Conference on Decision and Control
  (CDC)}. pp. 5004--5009.
\newblock \doi{10.1109/CDC.2016.7799034}.

\bibitem[{Bellman(1954)}]{bellman_ams54}
Bellman RE (1954) The theory of dynamic programming.
\newblock \emph{Bulletin of the American Mathematical Society (AMS)} 60(6):
  503--516.
\newblock \doi{10.1090/S0002-9904-1954-09848-8}.

\bibitem[{Bellman(1957)}]{bellman_57}
Bellman RE (1957) \emph{Dynamic Programming}.
\newblock Princeton University Press.
\newblock ISBN 978-0-691-07951-6.

\bibitem[{Bertsekas(1975)}]{bertsekas_tac75}
Bertsekas DP (1975) Convergence of discretization procedures in dynamic
  programming.
\newblock \emph{IEEE Transactions on Automatic Control (TAC)} 20(3): 415--419.
\newblock \doi{10.1109/TAC.1975.1100984}.

\bibitem[{Bohlin and Kavraki(2000)}]{bohlin_icra00}
Bohlin R and Kavraki LE (2000) Path planning using lazy {PRM}.
\newblock In: \emph{Proceedings of the IEEE International Conference on
  Robotics and Automation (ICRA)}, volume~1. pp. 521--528.
\newblock \doi{10.1109/ROBOT.2000.844107}.

\bibitem[{Branicky et~al.(2001)Branicky, LaValle, Olson and
  Yang}]{branicky_icra01}
Branicky MS, LaValle SM, Olson K and Yang L (2001) Quasi-randomized path
  planning.
\newblock In: \emph{Proceedings of the IEEE International Conference on
  Robotics and Automation (ICRA)}, volume~2. pp. 1481--1487.
\newblock \doi{10.1109/ROBOT.2001.932820}.

\bibitem[{Choudhury et~al.(2016)Choudhury, Gammell, Barfoot, Srinivasa and
  Scherer}]{choudhury_icra16}
Choudhury S, Gammell JD, Barfoot TD, Srinivasa SS and Scherer S (2016)
  {Regionally Accelerated Batch Informed Trees} ({RABIT*}): A framework to
  integrate local information into optimal path planning.
\newblock In: \emph{Proceedings of the IEEE International Conference on
  Robotics and Automation (ICRA)}. pp. 4207--4214.
\newblock \doi{10.1109/ICRA.2016.7487615}.

\bibitem[{Cohen et~al.(2014)Cohen, Phillips and Likhachev}]{cohen_rss14}
Cohen B, Phillips M and Likhachev M (2014) Planning single-arm manipulations
  with n-arm robots.
\newblock In: \emph{Proceedings of Robotics: Science and Systems (RSS)}.
  Berkeley, USA.
\newblock \doi{10.15607/RSS.2014.X.033}.

\bibitem[{Diankov and {Kuffner Jr.}(2007)}]{diankov_iros07}
Diankov R and {Kuffner Jr} JJ (2007) Randomized statistical path planning.
\newblock In: \emph{Proceedings of the IEEE/RSJ International Conference on
  Intelligent Robots and Systems (IROS)}.
\newblock \doi{10.1109/IROS.2007.4399557}.

\bibitem[{Dijkstra(1959)}]{dijkstra_59}
Dijkstra EW (1959) A note on two problems in connexion with graphs.
\newblock \emph{Numerische Mathematik} 1(1): 269--271.
\newblock \doi{10.1007/BF01386390}.

\bibitem[{Eppstein et~al.(1997)Eppstein, Paterson and Yao}]{eppstein_dcg97}
Eppstein D, Paterson MS and Yao FF (1997) On nearest-neighbor graphs.
\newblock \emph{Discrete \& Computational Geometry} 17(3): 263--282.
\newblock \doi{10.1007/PL00009293}.

\bibitem[{Euler(1738)}]{gamma_function}
Euler L (1738) De progressionibus transcendentibus seu quarum termini generales
  algebraice dari nequeunt.
\newblock \emph{Commentarii academiae scientiarum Petropolitanae} 5: 36--57.

\bibitem[{Ferguson and Stentz(2006)}]{ferguson_iros06}
Ferguson D and Stentz A (2006) Anytime {RRT}s.
\newblock In: \emph{Proceedings of the IEEE/RSJ International Conference on
  Intelligent Robots and Systems (IROS)}. pp. 5369--5375.
\newblock \doi{10.1109/IROS.2006.282100}.

\bibitem[{Gammell(2017)}]{gammell_phd17}
Gammell JD (2017) \emph{Informed Anytime Search for Continuous Planning
  Problems}.
\newblock PhD Thesis, University of Toronto.
\newblock \doi{1807/78630}.

\bibitem[{Gammell et~al.(2018)Gammell, Barfoot and Srinivasa}]{gammell_tro18}
Gammell JD, Barfoot TD and Srinivasa SS (2018) Informed sampling for
  asymptotically optimal path planning.
\newblock \emph{{IEEE} Transactions on Robotics ({T-RO})} 34(4): 966--984.
\newblock \doi{10.1109/TRO.2018.2830331}.

\bibitem[{Gammell et~al.(2014{\natexlab{a}})Gammell, Srinivasa and
  Barfoot}]{gammell_arxiv14c}
Gammell JD, Srinivasa SS and Barfoot TD (2014{\natexlab{a}}) {BIT*}: Batch
  informed trees for optimal sampling-based planning via dynamic programming on
  implicit random geometric graphs.
\newblock Technical Report {TR-2014-JDG006}, Autonomous Space Robotics Lab,
  University of Toronto.
\newblock \href{https://arxiv.org/abs/1405.5848v1}{arXiv:1405.5848v1 [cs.RO]}.

\bibitem[{Gammell et~al.(2014{\natexlab{b}})Gammell, Srinivasa and
  Barfoot}]{gammell_igmpw14}
Gammell JD, Srinivasa SS and Barfoot TD (2014{\natexlab{b}}) {BIT*}:
  Sampling-based optimal planning via batch informed trees.
\newblock In: \emph{The Information-based Grasp and Manipulation Planning
  Workshop, Robotics: Science and Systems (RSS)}. Berkeley, CA, USA.

\bibitem[{Gammell et~al.(2014{\natexlab{c}})Gammell, Srinivasa and
  Barfoot}]{gammell_iros14}
Gammell JD, Srinivasa SS and Barfoot TD (2014{\natexlab{c}}) Informed {RRT*}:
  Optimal sampling-based path planning focused via direct sampling of an
  admissible ellipsoidal heuristic.
\newblock In: \emph{Proceedings of the IEEE/RSJ International Conference on
  Intelligent Robots and Systems (IROS)}. pp. 2997--3004.
\newblock \doi{10.1109/IROS.2014.6942976}.

\bibitem[{Gammell et~al.(2015)Gammell, Srinivasa and Barfoot}]{gammell_icra15}
Gammell JD, Srinivasa SS and Barfoot TD (2015) {Batch Informed Trees} ({BIT*}):
  Sampling-based optimal planning via the heuristically guided search of
  implicit random geometric graphs.
\newblock In: \emph{Proceedings of the IEEE International Conference on
  Robotics and Automation (ICRA)}. pp. 3067--3074.
\newblock \doi{10.1109/ICRA.2015.7139620}.

\bibitem[{Gilbert(1961)}]{gilbert_siam61}
Gilbert EN (1961) Random plane networks.
\newblock \emph{Journal of the Society for Industrial and Applied Mathematics
  (JSIAM)} 9(4): 533--543.
\newblock \doi{10.1137/0109045}.

\bibitem[{Hart et~al.(1968)Hart, Nilsson and Raphael}]{hart_tssc68}
Hart PE, Nilsson NJ and Raphael B (1968) A formal basis for the heuristic
  determination of minimum cost paths.
\newblock \emph{IEEE Transactions on Systems Science and Cybernetics} 4(2):
  100--107.
\newblock \doi{10.1109/TSSC.1968.300136}.

\bibitem[{Hauser(2015)}]{hauser_icra15}
Hauser K (2015) Lazy collision checking in asymptotically-optimal motion
  planning.
\newblock In: \emph{Proceedings of the IEEE International Conference on
  Robotics and Automation (ICRA)}. pp. 2951--2957.
\newblock \doi{10.1109/ICRA.2015.7139603}.

\bibitem[{Helmert(2006)}]{helmert_jair06}
Helmert M (2006) The fast downward planning system.
\newblock \emph{Journal of Artificial Intelligence Research (JAIR)} 26:
  191--246.
\newblock \doi{10.1613/JAIR.1705}.

\bibitem[{Hsu et~al.(1999)Hsu, Latombe and Motwani}]{hsu_ijcga99}
Hsu D, Latombe JC and Motwani R (1999) Path planning in expansive configuration
  spaces.
\newblock \emph{International Journal of Computational Geometry \& Applications
  (IJCGA)} 9(4--5): 495--512.
\newblock \doi{10.1142/S0218195999000285}.

\bibitem[{Janson et~al.(2018)Janson, Ichter and Pavone}]{janson_ijrr18}
Janson L, Ichter B and Pavone M (2018) Deterministic sampling-based motion
  planning: Optimality, complexity, and performance.
\newblock \emph{The International Journal of Robotics Research (IJRR)} 37(1):
  46--61.
\newblock \doi{10.1177/0278364917714338}.

\bibitem[{Janson and Pavone(2013)}]{janson_isrr13}
Janson L and Pavone M (2013) Fast marching trees: a fast marching
  sampling-based method for optimal motion planning in many dimensions.
\newblock In: \emph{Proceedings of the International Symposium on Robotics
  Research (ISRR)}.
\newblock \doi{10.1007/978-3-319-28872-7_38}.

\bibitem[{Janson et~al.(2015)Janson, Schmerling, Clark and
  Pavone}]{janson_ijrr15}
Janson L, Schmerling E, Clark A and Pavone M (2015) Fast marching tree: A fast
  marching sampling-based method for optimal motion planning in many
  dimensions.
\newblock \emph{The International Journal of Robotics Research (IJRR)} 34(7):
  883--921.
\newblock \doi{10.1177/0278364915577958}.

\bibitem[{Karaman and Frazzoli(2011)}]{karaman_ijrr11}
Karaman S and Frazzoli E (2011) Sampling-based algorithms for optimal motion
  planning.
\newblock \emph{The International Journal of Robotics Research} 30(7):
  846--894.
\newblock \doi{10.1177/0278364911406761}.

\bibitem[{Karaman et~al.(2011)Karaman, Walter, Perez, Frazzoli and
  Teller}]{karaman_icra11}
Karaman S, Walter MR, Perez A, Frazzoli E and Teller S (2011) Anytime motion
  planning using the {RRT*}.
\newblock In: \emph{Proceedings of the IEEE International Conference on
  Robotics and Automation (ICRA)}. pp. 1478--1483.
\newblock \doi{10.1109/ICRA.2011.5980479}.

\bibitem[{Kavraki et~al.(1998)Kavraki, Kolountzakis and
  Latombe}]{kavraki_tra98}
Kavraki LE, Kolountzakis MN and Latombe JC (1998) Analysis of probabilistic
  roadmaps for path planning.
\newblock \emph{{IEEE} Transactions on Robotics and Automation} 14(1):
  166--171.
\newblock \doi{10.1109/70.660866}.

\bibitem[{Kavraki et~al.(1996)Kavraki, \v{S}vestka, Latombe and
  Overmars}]{kavraki_tro96}
Kavraki LE, \v{S}vestka P, Latombe JC and Overmars MH (1996) Probabilistic
  roadmaps for path planning in high-dimensional configuration spaces.
\newblock \emph{IEEE Transactions on Robotics and Automation} 12(4): 566--580.
\newblock \doi{10.1109/70.508439}.

\bibitem[{Kiesel et~al.(2012)Kiesel, Burns and Ruml}]{kiesel_socs12}
Kiesel S, Burns E and Ruml W (2012) Abstraction-guided sampling for motion
  planning.
\newblock In: \emph{Proceedings of the Fifth Annual Symposium on Combinatorial
  Search (SoCS)}.
\newblock ISBN 978-1-577-35584-7.

\bibitem[{Kleinbort et~al.(2015)Kleinbort, Salzman and
  Halperin}]{kleinbort_icra15}
Kleinbort M, Salzman O and Halperin D (2015) Efficient high-quality motion
  planning by fast all-pairs r-nearest-neighbors.
\newblock In: \emph{Proceedings of the IEEE International Conference on
  Robotics and Automation (ICRA)}. pp. 2985--2990.
\newblock \doi{10.1109/ICRA.2015.7139608}.

\bibitem[{Kleinbort et~al.(2016)Kleinbort, Salzman and
  Halperin}]{kleinbort_wafr16}
Kleinbort M, Salzman O and Halperin D (2016) Collision detection or
  nearest-neighbor search? {On} the computational bottleneck in sampling-based
  motion planning.
\newblock In: \emph{Proceedings of the International Workshop on the
  Algorithmic Foundations of Robotics (WAFR)}.

\bibitem[{Koenig et~al.(2004)Koenig, Likhachev and Furcy}]{koenig_ai04}
Koenig S, Likhachev M and Furcy D (2004) Lifelong planning {A*}.
\newblock \emph{Artificial Intelligence (AI)} 155(1--2): 93--146.
\newblock \doi{10.1016/J.ARTINT.2003.12.001}.

\bibitem[{{Kuffner Jr.} and LaValle(2000)}]{kuffner_icra00}
{Kuffner Jr} JJ and LaValle SM (2000) {RRT-Connect}: An efficient approach to
  single-query path planning.
\newblock In: \emph{Proceedings of the IEEE International Conference on
  Robotics and Automation (ICRA)}, volume~2. pp. 995--1001.
\newblock \doi{10.1109/ROBOT.2000.844730}.

\bibitem[{Kunz et~al.(2016)Kunz, Thomaz and Christensen}]{kunz_icra16}
Kunz T, Thomaz A and Christensen H (2016) Hierarchical rejection sampling for
  informed kinodynamic planning in high-dimensional spaces.
\newblock In: \emph{Proceedings of the IEEE International Conference on
  Robotics and Automation (ICRA)}. pp. 89--96.
\newblock \doi{10.1109/ICRA.2016.7487120}.

\bibitem[{LaValle and {Kuffner Jr.}(2001)}]{lavalle_ijrr01}
LaValle SM and {Kuffner Jr} JJ (2001) Randomized kinodynamic planning.
\newblock \emph{The International Journal of Robotics Research} 20(5):
  378--400.
\newblock \doi{10.1177/02783640122067453}.

\bibitem[{Likhachev et~al.(2008)Likhachev, Ferguson, Gordon, Stentz and
  Thrun}]{likhachev_ai08}
Likhachev M, Ferguson D, Gordon G, Stentz A and Thrun S (2008) Anytime search
  in dynamic graphs.
\newblock \emph{Artificial Intelligence (AI)} 172(14): 1613--1643.
\newblock \doi{http://dx.doi.org/10.1016/j.artint.2007.11.009}.

\bibitem[{Lozano-P{\'e}rez(1983)}]{lozano-perez_tc83}
Lozano-P{\'e}rez T (1983) Spatial planning: A configuration space approach.
\newblock \emph{IEEE Transactions on Computers} C-32(2): 108--120.
\newblock \doi{10.1109/TC.1983.1676196}.

\bibitem[{Muthukrishnan and Pandurangan(2005)}]{muthukrishnan_siam05}
Muthukrishnan S and Pandurangan G (2005) The bin-covering technique for
  thresholding random geometric graph properties.
\newblock In: \emph{Proceedings of the Sixteenth Annual ACM-SIAM Symposium on
  Discrete Algorithms}.
\newblock ISBN 978-0-898-71585-9, pp. 989--998.

\bibitem[{Otte and Frazzoli(2014)}]{otte_wafr14}
Otte M and Frazzoli E (2014) {RRT\textsuperscript{X}}: Real-time motion
  planning/replanning for environments with unpredictable obstacles.
\newblock In: \emph{International Workshop on the Algorithmic Foundations of
  Robotics (WAFR)}. Istanbul, Turkey.

\bibitem[{Otte and Frazzoli(2016)}]{otte_ijrr16}
Otte M and Frazzoli E (2016) {RRT\textsuperscript{X}}: Asymptotically optimal
  single-query sampling-based motion planning with quick replanning.
\newblock \emph{The International Journal of Robotics Research (IJRR)} 35(7):
  797--822.
\newblock \doi{10.1177/0278364915594679}.

\bibitem[{Penrose(2003)}]{penrose_03}
Penrose M (2003) \emph{Random Geometric Graphs}, \emph{Oxford Studies in
  Probability}, volume~5.
\newblock Oxford University Press.
\newblock ISBN 978-0-198-50626-3.

\bibitem[{Perez et~al.(2011)Perez, Karaman, Shkolnik, Frazzoli, Teller and
  Walter}]{perez_iros11}
Perez A, Karaman S, Shkolnik A, Frazzoli E, Teller S and Walter MR (2011)
  Asymptotically-optimal path planning for manipulation using incremental
  sampling-based algorithms.
\newblock In: \emph{Proceedings of the IEEE/RSJ International Conference on
  Intelligent Robots and Systems (IROS)}. pp. 4307--4313.
\newblock \doi{10.1109/IROS.2011.6094994}.

\bibitem[{Persson and Sharf(2014)}]{persson_ijrr14}
Persson SM and Sharf I (2014) Sampling-based {A*} algorithm for robot
  path-planning.
\newblock \emph{The International Journal of Robotics Research (IJRR)} 33(13):
  1683--1708.
\newblock \doi{10.1177/0278364914547786}.

\bibitem[{Sallaberger and D'Eleuterio(1995)}]{sallaberger_acta95}
Sallaberger CS and D'Eleuterio GM (1995) Optimal robotic path planning using
  dynamic programming and randomization.
\newblock \emph{Acta Astronautica} 35(2--3): 143--156.
\newblock \doi{10.1016/0094-5765(94)00158-I}.

\bibitem[{Salzman and Halperin(2015)}]{salzman_icra15}
Salzman O and Halperin D (2015) Asymptotically-optimal motion planning using
  lower bounds on cost.
\newblock In: \emph{Proceedings of the {IEEE} International Conference on
  Robotics and Automation ({ICRA})}. pp. 4167--4172.
\newblock \doi{10.1109/ICRA.2015.7139773}.

\bibitem[{Salzman and Halperin(2016)}]{salzman_tro16}
Salzman O and Halperin D (2016) Asymptotically near-optimal {RRT} for fast,
  high-quality motion planning.
\newblock \emph{{IEEE} Transactions on Robotics} 32(3): 473--483.
\newblock \doi{10.1109/TRO.2016.2539377}.

\bibitem[{S\'{a}nchez and Latombe(2002)}]{sanchez_ijrr02}
S\'{a}nchez G and Latombe JC (2002) On delaying collision checking in {PRM}
  planning: Application to multi-robot coordination.
\newblock \emph{The International Journal of Robotics Research (IJRR)} 21(1):
  5--26.
\newblock \doi{10.1177/027836402320556458}.

\bibitem[{Srinivasa et~al.(2012)Srinivasa, Berenson, Cakmak, {Collet Romea},
  Dogar, Dragan, Knepper, Niemueller, Strabala, Vandeweghe and Ziegler}]{herb}
Srinivasa S, Berenson D, Cakmak M, {Collet Romea} A, Dogar M, Dragan A, Knepper
  RA, Niemueller TD, Strabala K, Vandeweghe JM and Ziegler J (2012) {HERB} 2.0:
  Lessons learned from developing a mobile manipulator for the home.
\newblock \emph{Proceedings of the IEEE} 100(8): 1--19.
\newblock \doi{10.1109/JPROC.2012.2200561}.

\bibitem[{{\c{S}}ucan et~al.(2012){\c{S}}ucan, Moll and Kavraki}]{ompl}
{\c{S}}ucan IA, Moll M and Kavraki LE (2012) The {O}pen {M}otion {P}lanning
  {L}ibrary.
\newblock \emph{{IEEE} Robotics \& Automation Magazine} 19(4): 72--82.
\newblock \doi{10.1109/MRA.2012.2205651}.
\newblock Library available from
  {\footnotesize\url{http://ompl.kavrakilab.org/}}.

\bibitem[{Teniente and Andrade-Cetto(2013)}]{teniente_iros13}
Teniente EH and Andrade-Cetto J (2013) {HRA*}: Hybrid randomized path planning
  for complex {3D} environments.
\newblock In: \emph{Proceedings of the IEEE/RSJ International Conference on
  Intelligent Robots and Systems (IROS)}. pp. 1766--1771.
\newblock \doi{10.1109/IROS.2013.6696588}.

\bibitem[{Urmson and Simmons(2003)}]{urmson_iros03}
Urmson C and Simmons R (2003) Approaches for heuristically biasing {RRT}
  growth.
\newblock In: \emph{Proceedings of the IEEE/RSJ International Conference on
  Intelligent Robots and Systems (IROS)}, volume~2. pp. 1178--1183.
\newblock \doi{10.1109/IROS.2003.1248805}.

\bibitem[{Xie et~al.(2015)Xie, {van den Berg}, Patil and Abbeel}]{xie_icra15}
Xie C, {van den Berg} J, Patil S and Abbeel P (2015) Toward asymptotically
  optimal motion planning for kinodynamic systems using a two-point boundary
  value problem solver.
\newblock In: \emph{Proceedings of the IEEE International Conference on
  Robotics and Automation (ICRA)}. pp. 4187--4194.
\newblock \doi{10.1109/ICRA.2015.7139776}.

\bibitem[{Yershov and Frazzoli(2016)}]{yershov_ijrr16}
Yershov DS and Frazzoli E (2016) Asymptotically optimal feedback planning using
  a numerical {Hamilton}-{Jacobi}-{Bellman} solver and an adaptive mesh
  refinement.
\newblock \emph{The International Journal of Robotics Research (IJRR)} 35(5):
  565--584.
\newblock \doi{10.1177/0278364915602958}.

\bibitem[{Zucker et~al.(2013)Zucker, Ratliff, Dragan, Pivtoraiko, Klingensmith,
  Dellin, Bagnell and Srinivasa}]{zucker_ijrr13}
Zucker M, Ratliff N, Dragan AD, Pivtoraiko M, Klingensmith M, Dellin CM,
  Bagnell JA and Srinivasa SS (2013) {CHOMP}: Covariant {Hamiltonian}
  optimization for motion planning.
\newblock \emph{The International Journal of Robotics Research (IJRR)}
  32(9--10): 1164--1193.
\newblock \doi{10.1177/0278364913488805}.

\end{thebibliography}
\end{document}